\documentclass[10pt, english, a4paper]{article}

\usepackage[utf8]{inputenc} 
\PassOptionsToPackage{hyphens}{url}
\usepackage[hidelinks]{hyperref}
\usepackage{graphicx}
\usepackage{varioref, natbib, listings, amsmath,amsthm, amssymb, mathtools}
\usepackage{array}
\usepackage{rotating}
\usepackage[a4paper, total={6in, 8in}]{geometry}
\usepackage{enumitem}
\usepackage{multicol}
\usepackage{aliascnt} 

\usepackage[capitalize,nameinlink]{cleveref}[0.19]
\setlist[enumerate]{leftmargin=.5in}
\setlist[itemize]{leftmargin=.5in}

\renewcommand{\hat}{\widehat}
\renewcommand{\tilde}{\widetilde}

\Crefname{figure}{Figure}{Figures}

\crefformat{equation}{\textup{#2(#1)#3}}
\crefrangeformat{equation}{\textup{#3(#1)#4--#5(#2)#6}}
\crefmultiformat{equation}{\textup{#2(#1)#3}}{ and \textup{#2(#1)#3}}
{, \textup{#2(#1)#3}}{, and \textup{#2(#1)#3}}
\crefrangemultiformat{equation}{\textup{#3(#1)#4--#5(#2)#6}}%
{ and \textup{#3(#1)#4--#5(#2)#6}}{, \textup{#3(#1)#4--#5(#2)#6}}{, and \textup{#3(#1)#4--#5(#2)#6}}

\Crefformat{equation}{#2Equation~\textup{(#1)}#3}
\Crefrangeformat{equation}{Equations~\textup{#3(#1)#4--#5(#2)#6}}
\Crefmultiformat{equation}{Equations~\textup{#2(#1)#3}}{ and \textup{#2(#1)#3}}
{, \textup{#2(#1)#3}}{, and \textup{#2(#1)#3}}
\Crefrangemultiformat{equation}{Equations~\textup{#3(#1)#4--#5(#2)#6}}%
{ and \textup{#3(#1)#4--#5(#2)#6}}{, \textup{#3(#1)#4--#5(#2)#6}}{, and \textup{#3(#1)#4--#5(#2)#6}}

\crefdefaultlabelformat{#2\textup{#1}#3}

\usepackage{mathtools, multicol, makecell}

\theoremstyle{plain}
\newtheorem{theorem}{Theorem}[section]
\newtheorem{lemma}[theorem]{Lemma}
\newtheorem{proposition}[theorem]{Proposition}
\newtheorem{corollary}[theorem]{Corollary}

\theoremstyle{definition}
\newtheorem{definition}[theorem]{Definition}
\newtheorem{example}[theorem]{Example}
\newtheorem{assumption}[theorem]{Assumption}
\crefname{assumption}{assumption}{assumptions}

\theoremstyle{remark}
\newtheorem{remark}[theorem]{Remark}

\renewcommand{\hat}{\widehat}
\renewcommand{\tilde}{\widetilde}

\newcommand\Z{\mathbb{Z}}

\newcommand{\R}{\mathbb{R}}
\newcommand{\E}{\mathbb{E}}
\newcommand{\T}{^{\top}}
\newcommand{\norm}[1]{\left\|#1\right\|}

\DeclareMathOperator{\rank}{rank}

\newcommand{\Bt}[1]{W_#1\T J W_#1}
\newcommand\A{\mathcal{A}}
\newcommand\PA{P_A}
\newcommand\PP{P_\perp}
\newcommand\PN{P_0}
\newcommand\dt{\frac{d}{dt}}

\title{Implicit Regularization Makes Overparameterized Asymmetric Matrix Sensing Robust to Perturbations}

\author{Johan S. Wind\thanks{Department of Mathematics, University of Oslo, Norway (\href{mailto:johanswi@math.uio.no}{johanswi@math.uio.no})}}

\begin{document}

\maketitle

\begin{abstract}%
  Several key questions remain unanswered regarding overparameterized learning models. It is unclear how (stochastic) gradient descent finds solutions that generalize well, and in particular the role of small random initializations. Matrix sensing, which is the problem of reconstructing a low-rank matrix from a few linear measurements, has become a standard prototypical setting to study these phenomena. Previous works have shown that matrix sensing can be solved by factorized gradient descent, provided the random initialization is extremely small.

  In this paper, we find that factorized gradient descent is highly robust to certain perturbations. This lets us use a perturbation term to capture both the effects of imperfect measurements, discretization by gradient descent, and other noise, resulting in a general formulation which we call \textit{perturbed gradient flow}. We find that not only is this equivalent formulation easier to work with, but it leads to sharper sample and time complexities than previous work, handles moderately small initializations, and the results are naturally robust to perturbations such as noisy measurements or changing measurement matrices. Finally, we also analyze mini-batch stochastic gradient descent using the formulation, where we find improved sample complexity.
\end{abstract}


\section{Introduction}
While modern machine learning methods show excellent empirical performance, several fundamental questions regarding the understanding of these methods remain open. Especially troubling is the widespread use of \textit{overparameterization}, which lets these methods fit arbitrary data perfectly \citep{zhang2016understanding}. This causes most classical data-independent generalization bounds to be vacuous \citep{dziugaite2017computing}. There seems to be some \textit{implicit regularization} (also called algorithmic regularization or implicit bias), which leads to solutions which generalize well.

Directly studying these questions for practical neural networks is largely analytically intractable. However, the study of linear neural networks has led to a greater understanding of key mechanisms of implicit regularization. In linear neural networks, it is clear that optimization by (stochastic) gradient descent, coupled with a small initialization, leads to implicit regularization towards solutions with "small rank". The meaning of "rank" is only clearly understood in concrete settings, such as matrix sensing/completion \citep{arora_factorization}, tensor sensing/completion \citep{razin2021implicit} and compressed sensing\footnote{In compressed sensing the analogue to small rank is sparsity.} \citep{wind2023implicit}. However, empirical experiments indicate that the general intuition carries over to more realistic settings \citep{boix2023transformers}.

A particularly interesting setting is the setting of matrix sensing. In matrix sensing, we are tasked with reconstructing an $n_1\times n_2$ low-rank matrix $Y$ from $m \ll n_1n_2$ linear measurements $\A(Y)$. We can perfectly reconstruct the matrix, for example by nuclear norm minimization \citep{recht2010guaranteed}, if the measurement operator $\A$ satisfies RIP (Restricted Isometry Property, \Cref{def:rip}). There are several sophisticated algorithms which can reconstruct $Y$ in this setting. However, a particularly intriguing algorithm is the following:

Optimize the following objective by gradient descent:
\begin{align}\label{fact_gd}
  L(U,V) = \frac{1}{2}\norm{\A(Y)-\A(UV\T)}^2,
\end{align}
with the matrices $U$ and $V$ initialized elementwise to $\mathcal{N}(0,\epsilon^2)$ for some small $\epsilon > 0$.

This approach is called factorized gradient descent (from small random initialization), and works empirically both for matrix sensing \citep{rip_asymmetric} and matrix completion \citep{arora_factorization}. The problem shares several similarities with the effects of implicit regularization found in practical neural networks: First, factorized gradient decent works well even in the presence of severe overparameterization. Second, it successfully and quickly optimizes a non-convex objective function. Third, it implicitly exploits structure in the data (here the small rank of $Y$) without any explicit regularizer or constraint.

Because of these intriguing similarities to modern neural networks, there has been much work on studying factorized gradient descent. However, analyzing \eqref{fact_gd} directly is difficult, so most previous works consider simplified settings. A common simplification is to consider the case $U = V$, which is called the symmetric PSD (Positive Semi-Definite) matrix sensing. Specifically, replace \eqref{fact_gd} with $L(U) = \frac{1}{2}\norm{\A(Y)-\A(UU\T)}^2$ and optimize by gradient decent. Clearly, since $UU\T$ is symmetric PSD, this can only reconstruct symmetric PSD matrices $Y$. Another common simplification is to consider perfect information, that is, gradient descent on $L(U,V) = \frac{1}{2}\norm{Y-UV\T}_F^2$, which is called matrix factorization. Finally, we may simplify the analysis by considering gradient flow instead of gradient decent, by taking the limit of infinitesimal learning rate. Of course, it is possible to consider combinations of the preceding simplifications, like matrix factorization in the symmetric PSD setting, that is, gradient descent on $L(U,U) = \frac{1}{2}\norm{Y-UU\T}_F^2$. A classification of 15 selected previous works in these settings can be found in \Cref{table}.

This work tackles the difficult setting of asymmetric matrix sensing with imperfect measurements, without simplifications. Our main contributions are as follows:
\begin{itemize}
  \item We provide a full and standalone proof that factorized gradient descent successfully reconstructs low-rank matrices from a small random initialization, even in the severely overparameterized setting. Compared to the only previous work in this setting, \citet{rip_asymmetric}, we have several advantages:
    \begin{itemize}
      \item Our proof method makes our reconstruction results naturally robust to arbitrary perturbations during the optimization process. This includes, for example, noisy measurements or adversarially changing the measurement operator $\A$ in each iteration of gradient descent.
      \item Our result is sharper both in terms of number of samples required and number of gradient descent iterations required.
      \item They require impractically small initializations (exponentially small in the condition number $\kappa \coloneqq \norm{Y}/\sigma_{min}(Y)$), our results also hold for realistic, moderately small initializations.
    \end{itemize}
  \item We also consider mini-batch stochastic gradient descent, where we find near-optimal sample complexity. This is significant since it improves on the sample complexities of factorized gradient descent from random initializations which use the Restricted Isometry Property (\Cref{def:rip}), which consistently require an extra factor $\rank(Y)$ in the sample complexity \citep{rip_symmetric,rip_small_init,ding2022validation,glrl2,scaledGD,rip_asymmetric}.
  \item Our general proof technique, which we call reduction to \textit{perturbed gradient flow}, is novel. We believe it may prove useful in other settings as well.
\end{itemize}

\textbf{Notations.} We use $X_t$ to denote the value of the variable $X$ at time $t$. Specifically, $X_0$ is the value of $X$ at initialization. We sometimes drop $t$ dependencies in statements that hold at all times. We write $\sigma_k(X)$ for the $k$th largest singular value of the matrix $X$. Furthermore, $\norm{\cdot}$ denotes the Euclidean norm for vectors and the operator norm for matrices. The Frobenius norm is denoted $\norm{\cdot}_F$. Finally, $\A^*$ denotes the adjoint of the operator $\A$.

\section{Main results}\label{s:main_results}
We generalize factorized gradient descent \eqref{fact_gd} as follows: First, we allow the measurement operator $\A$ to change arbitrarily between iterations. Second, we add an arbitrary (possibly adversarial) perturbation $\tilde E_k$ in each iteration, and only require an upper bound on $\sup_{k\ge0}\norm{\tilde E_k}$.
\begin{definition}[Perturbed Gradient Descent]\label{pert_gd}
  We say that matrices $U$ and $V$ follow perturbed gradient descent with learning rate $\eta > 0$ and perturbations $\{\tilde E\}_{k\ge0}$ if
  \begin{equation}\label{eq:perturbed_L}
    \begin{split}
      \begin{pmatrix}U_{(k+1)\eta}\\V_{(k+1)\eta}\end{pmatrix} &= \begin{pmatrix}U_{k\eta}\\V_{k\eta}\end{pmatrix} - \eta \nabla L_k\left[\begin{pmatrix}U_{k\eta}\\V_{k\eta}\end{pmatrix}\right] + \eta\tilde E_k \begin{pmatrix}U_{k\eta}\\V_{k\eta}\end{pmatrix},\\
      L_k\left[\begin{pmatrix}U\\V\end{pmatrix}\right] &= \frac{1}{2}\norm{\A_k(Y)-\A_k(UV\T)}^2.
    \end{split}
  \end{equation}
\end{definition}
\begin{remark}
  We use the notation $U_{k\eta}$ for the $k$th iterate of $U$. This is because our proofs work with a continuous interpolation $U_t,\ t\in\R_{\ge0}$ of the iterates $U_{k\eta},\ k \in \Z_{\ge0}$. The notation $U_{k\eta}$ is then consistent with both settings. See \Cref{s:pert_flow} for more details.
\end{remark}

\begin{remark}
  It might be instructive to see the iteration written out in terms of $U$ and $V$:
  \begin{equation}\label{eq:perturbed_gd}
    \begin{split}
      U_{(k+1)\eta} &= U_{k\eta}+\eta(\A_k^*\A_k)(Y-U_{k\eta}V_{k\eta}\T)\,\,\ V_{k\eta} + \eta[\tilde E_{UU}]_k U_{k\eta} + \eta[\tilde E_{UV}]_k V_{k\eta},\\
      V_{(k+1)\eta} &= V_{k\eta}+\eta(\A_k^*\A_k)(Y-U_{k\eta}V_{k\eta}\T)\T U_{k\eta} + \eta[\tilde E_{VU}]_k U_{k\eta} + \eta[\tilde E_{VV}]_k V_{k\eta}.
    \end{split}
  \end{equation}
  Here we split $\tilde E_k = \begin{pmatrix}[\tilde E_{UU}]_k & [\tilde E_{UV}]_k \\ [\tilde E_{VU}]_k & [\tilde E_{VV}]_k\end{pmatrix}$. Note that we can bound $\norm{\tilde E_k}$ by bounding each of its components. Specifically, $\norm{\tilde E_k} \le \norm{[\tilde E_{UU}]_k}+\norm{[\tilde E_{UV}]_k}+\norm{[\tilde E_{VU}]_k}+\norm{[\tilde E_{VV}]_k}$.
\end{remark}
\begin{example}
  The perturbation terms $\{\tilde E\}_{k\ge0}$ can be used to model the effect of noise in the measurements. Following \citet{ding2022validation}, we may consider noisy measurements $\A(Y)+e$, where $e \in \R^m$ is additive noise. This leads to gradient descent on the modified loss function $L(U,V) = \frac{1}{2}\norm{\A(Y) + e - \A(UV\T)}^2$. The modified dynamics can be captured by a constant perturbation with $\norm{\tilde E_k} = \norm{\A^*(e)}$. \citet{ding2022validation} derives bounds on $\norm{\A^*(e)}$. Such bounds can be used directly to bound $\xi$ in \Cref{rip_theorem}, yielding results comparable to \citet{ding2022validation}, although their results are restricted to the symmetric PSD setting.
\end{example}

We will require bounded misalignment at initialization, which is defined as follows:
\begin{definition}[Misalignment at initialization]\label{def:pert_init}
  Let $Y = U_Y \Lambda_Y V_Y\T$ be a compact singular value decomposition of $Y$. Then the misalignment $\alpha \ge 1$ at the initialization $U_0,V_0$ is defined as
  \begin{align*}
    \alpha = \frac{\sqrt 2 \norm{\begin{pmatrix}U_0\T & V_0\T\end{pmatrix}}}{\sigma_r\left(U_Y\T U_0 + V_Y\T V_0\right)}.
  \end{align*}
\end{definition}
\begin{remark}\label{remark:init}
  To get initial misalignment $\alpha \le C$ and initialization scale $\norm{\begin{pmatrix}U_0\T & V_0\T\end{pmatrix}} \le \epsilon$, we may initialize as follows: Sample the entries of $U_0 \in \R^{n_1\times h}$ and $V_0 \in \R^{n_2\times h}$ independently from a normal distribution with the small standard deviation $\frac{\epsilon}{C\sqrt h} > 0$. Given a sufficiently large universal constant $C$, and that the number of hidden neurons satisfies $h \ge \max(n_1,n_2,2r)$, we have $\alpha \le C$ with high probability. We omit the proof, which applies standard tail bounds for gaussian matrices, for example, \citep[Corollary 7.3.3 and Exercise 7.3.4]{vershynin2018high}. A similar proof can also be found in \citep[Lemma B.7]{rip_asymmetric}. As an alternative initialization, we note that if $n_1 = n_2 = h$, then the scaled identity initialization, $U_0 = V_0 = \frac{\epsilon}{\sqrt 2} I$, has minimal initial misalignment $\alpha = 1$.
\end{remark}

\subsection{Robust matrix sensing with RIP measurements}\label{s:rip_result}
Our first result will require the measurement operators to satisfy the following property:
\begin{definition}[RIP: Restricted Isometry Property]\label{def:rip}
  The operator $\A \colon \R^{n_1\times n_2} \to \R^d$ satisfies RIP with rank $r$ and constant $\rho$, if for all matrices $X \in \R^{n_1\times n_2}$ with $\rank(X) \le r$, we have
  \begin{align}\label{eq:RIP}
    (1-\rho)\norm{X}_F^2 \le \norm{\A(X)}^2 \le (1+\rho)\norm{X}_F^2.
  \end{align}
\end{definition}
If the measurement operator $\A$ consists of $\frac{Cr(n_1+n_2)}{\rho^2}$ random gaussian measurements (for a sufficiently large universal constant $C$), it satisfies RIP with rank $r$ and constant $\rho$ with high probability \citep{candes2011tight}. Hence, the following \Cref{rip_theorem} immediately yields an algorithm for asymmetric matrix sensing with $O((n_1+n_2)r^2 \kappa^4)$ random gaussian measurements.

\begin{theorem}\label{rip_theorem}
  Let the target matrix $Y \in \R^{n_1\times n_2}$ have $\mathrm{rank}(Y) = r \ge 1$ and condition number $\kappa = \frac{\norm{Y}}{\sigma_r(Y)}$. Let $U \colon \R_{\ge0} \to \R^{n_1\times h}$ and $V \colon \R_{\ge0} \to \R^{n_2\times h}$ follow perturbed gradient descent (\Cref{pert_gd}) with learning rate $\eta$ and perturbations $\{\tilde E_k\}_{k\ge0}$. Assume the measurement operators $\{\A_k\}_{k\ge0}$ satisfy the Restricted Isometry Property with rank $r+1$ and constant $\rho \le c_1/(\sqrt r\kappa^2)$.

  Also, assume the misalignment at initialization (\Cref{def:pert_init}, see \Cref{remark:init}) satisfies $\alpha \le C_1$. Furthermore, assume the scale of initialization $\epsilon$ satisfies \[\epsilon \coloneqq \norm{\begin{pmatrix}U_0\T & V_0\T\end{pmatrix}} \le \frac{c_2\sqrt{\sigma_r(Y)}}{\min(n_1,n_2)^2/r^2+\kappa}.\]

  The learning rate $\eta$ is assumed to satisfy
  \[\eta \le \frac{c_3}{\kappa^3 \norm{Y}\log(\sigma_r(Y)/\epsilon^2)}.\]

  Finally, assume the perturbations $\{\tilde E_k\}_{k\ge0}$ satisfy \[\xi \coloneqq \sup_{k\ge 0}\norm{\tilde E_k} \le \frac{c_4}{\kappa^3 \log(\sigma_r(Y)/\epsilon^2)}\norm{Y}.\]

  Then after $K = \frac{C_2}{\eta \sigma_r(Y)}\log\left(\frac{\sigma_r(Y)}{\epsilon^2}\right)$ steps of gradient descent, we have
  \begin{align*}
    \norm{Y-U_{K\eta} V_{K\eta}\T} \lesssim \left(\frac{\epsilon^2}{\sigma_r(Y)}\right)^\frac{2}{3} \sigma_r(Y) + \kappa \xi.
  \end{align*}
  The constants $c_1,c_2,c_3,c_4,C_1,C_2 > 0$ are universal.
\end{theorem}

A few comments are in order.

\textbf{Reconstruction accuracy:} The term $\left(\frac{\epsilon^2}{\sigma_r(Y)}\right)^\frac{2}{3} \sigma_r(Y)$ can be made arbitrarily small by choosing a small initialization scale $\epsilon$. Hence, we may reconstruct $Y$ to arbitrary precision in the noiseless case ($\xi = 0$). The constant $\frac{2}{3}$ was chosen arbitrarily for simplicity of presentation; it can be changed by selecting different constants $c_1,\dots,C_2$ and initialization scale $\epsilon$. See \Cref{master_theorem} for more detailed dependencies.

\textbf{Overparameterization:} The number of parameters in $U$ and $V$ is $(n_1+n_2)h$, where $h$ is the number of hidden neurons. Note that the result is independent of $h$, which makes it hold even for arbitrary amounts of over-parameterization.

\textbf{Comparison with \citet{rip_asymmetric}:} The previous work most comparable to \Cref{rip_theorem} is \citet{rip_asymmetric}. They do not allow any perturbations $\tilde E$, not even noise in the measurements or changing measurement operators $\A_k$. They are also dependent on extremely small initializations, $\epsilon \le \sqrt{\norm{Y}}\left(\frac{1}{10\kappa}\right)^{68\kappa}$, which leads to the so-called \text{alignment phase}. This is impractical in implementation, as it would require extended precision for even moderate condition numbers $\kappa$, and it is entirely unnecessary empirically. They require RIP with rank $2r+1$ and constant $\rho \le c_1/(\kappa^3\sqrt r)$, leading to a sample complexity of $O((n_1+n_2)r^2\kappa^6)$. Our corresponding sample complexity is sharper: $O((n_1+n_2)r^2\kappa^4)$. Ignoring log factors, they require on the order of $\kappa^8$ steps of gradient descent, while we require only $\kappa^4$. It is worth noting that the extremely small initialization allows them to use the alignment phase to handle arbitrary misalignment at the initialization, which allows any number of hidden neurons $h \ge r$. We require a bounded misalignment at the initialization, which essentially requires $h \gtrsim \max(n_1,n_2)$. It is not clear whether \Cref{rip_theorem} holds for fewer hidden neurons, without resorting to extremely small initializations.

\textbf{Robustness:} In terms of robustness to perturbations, the closest work is \citet{ding2022validation}. Their work is in the simplified setting of symmetric PSD matrix sensing. They also assume that the perturbation is fixed throughout training. However, their final bound agrees well with ours. Both essentially find an additive error proportional to $\kappa$ times the operator norm of the perturbation.

\subsection{Robust matrix sensing with mini-batch stochastic gradient descent}\label{s:sgd_result}
We can achieve better sample complexity if instead of assuming the measurements satisfy RIP, we sample in each iteration a set of random linear measurements.
\begin{theorem}\label{sgd_theorem}
  Consider the setting of \Cref{rip_theorem}, but instead of the Restricted Isometry Property, assume that the measurement operators $\A_k \colon \R^{n_1\times n_2} \to \R^m$ are independent and of the form $[\A_k(X)]_i = \frac{1}{\sqrt m} \langle X, [A_k]_i\rangle$, for $i = 1,\dots,m$, where $[A_k]_i \in \R^{n_1\times n_2}$ has i.i.d. entries from $\mathcal{N}(0,1)$. Assume the mini-batch size $m$ satisfies
  \[m \ge C (\log(n_1+n_2)+\log(K)+\log(1/\delta)) r(n_1+n_2) \kappa^4,\]
  where $\delta \in (0,1)$ is the probability of failure and $C$ is a universal constant. Then, with probability at least $1-\delta$, we have
  \begin{align*}
    \norm{Y-U_{K\eta} V_{K\eta}\T} \lesssim \left(\frac{\epsilon^2}{\sigma_r(Y)}\right)^\frac{2}{3} \sigma_r(Y) + \kappa \xi.
  \end{align*}
\end{theorem}

\textbf{Sample complexity:} Note that we could get a result similar to \Cref{sgd_theorem} by taking enough samples in each mini-batch for each $\A_k$ to satisfy RIP. Then \Cref{rip_theorem} would imply successful reconstruction of $Y$. However, even for well-conditioned matrices ($\kappa \lesssim 1$), that would require on the order of $r^2(n_1+n_2)$ samples per mini-batch, which is suboptimal. However, \Cref{sgd_theorem} only requires $\tilde O(r(n_1+n_2))$ samples in total, which is optimal up to log factors. We believe this is the first such result with near-optimal sample complexity without resorting to specialized initializations.

\section{Related work}\label{s:related}

\begin{table}
  \begin{center}
    \begin{tabular}{|c| c c |}
      \hline
    & Symmetric & Asymmetric \\
    \hline
      Gradient flow & \citet{du_asymmetric}$^\dagger$ & \makecell{\citet{flow2};\\\citet{factorization_flow}} \\
      \hline
      Gradient descent & \makecell{\citet{jain2017global};\\\citet{chou2024gradient}} & \makecell{\citet{bartlett2018gradient};\\\citet{balanced};\\\citet{nguegnang2021convergence};\\\citet{du_asymmetric};\\\citet{sharper_asymmetric}} \\
      \hline
      RIP measurements & \makecell{\citet{rip_symmetric};\\\citet{rip_small_init};\\\citet{ding2022validation};\\\citet{glrl2};\\\citet{scaledGD}} & \makecell{\citet{rip_asymmetric};\\(This work)} \\
      \hline
    \end{tabular}
  \end{center}
  \caption{Selected previous works on matrix factorization and matrix sensing by factorized gradient descent from a small initialization. $\ ^\dagger$\citet{du_asymmetric} mainly focuses on asymmetric matrix factorization, but include a derivation of the closed form solution for the dynamics of symmetric PSD factorization under gradient flow.}\label{table}
\end{table}
There is a wealth of research into low-rank matrix factorization via nonconvex optimization. For an overview, see \citet{chi2019nonconvex}. Most works either perform a global landscape analysis or directly analyze the dynamics of the optimization algorithm. Works on landscape analysis prove benign properties of the loss landscape; for example \citet{bhojanapalli2016global,zhu2021global}. Typically, they show that all critical points of the optimization landscape are either global minima or strict saddle points, which allows global optimization by standard local search algorithms. Our work analyzes the gradient descent dynamics directly and exploits properties that only hold close to the gradient descent path, leading to a more fine-grained analysis.

Most algorithms based on gradient descent for low-rank matrix sensing and matrix completion, use specialized initializations (usually "spectral initialization"). This simplifies the analysis, since the initialization is already close to the global minimum. This effectively allows them to only provide proof of local convergence. We highlight some examples: \citet{local_rip_gd,local_rip_gd_regB,local_rip_gd_noregB,sgd_completion}. In the rest of the section, we focus on dynamics-based works starting from small initializations.

It is natural to classify works in two dimensions according to which simplifying assumptions they make. First, some works assume that the matrix to be reconstructed is symmetric PSD (Positive Semi-Definite). Second, they might study so-called matrix factorization, where perfect measurements $\A = I$ are assumed. There has also been work on matrix factorization optimized by gradient flow instead of gradient descent. Selected works in these settings are classified in \Cref{table}.

\citet{rip_symmetric} considers the case of symmetric PSD matrix sensing with RIP measurements. They start from a small orthogonal initialization, and prove convergence by inductively proving several bounds. The subsequent work of \citet{rip_small_init} sharpens their result. \citet{rip_small_init} also introduces the idea of using a very small random initialization (exponentially small in the condition number $\kappa$), which allows their result to apply to cases with fewer hidden neurons (our $h$) compared to \citet{rip_symmetric}. Building on this work, \citet{ding2022validation} shows robustness to random measurement noise.

For the case of asymmetric matrix factorization, the main added difficulty over the symmetric case is handling the "imbalance" $\norm{U\T U-V\T V}$. In the case of gradient flow, \citet[Theorem 2.2]{balanced} shows that the imbalance is constant, and hence stays small for small initializations. \citet{du_asymmetric} gives an elegant proof that asymmetric matrix factorization is solvable by vanilla gradient descent. \citet{sharper_asymmetric} gives a sharper result in the same setting.

The recent work of \citet{rip_asymmetric} considers the problem of asymmetric matrix sensing with RIP measurements, a setting similar to ours. They follow \citet{rip_small_init}, and assume an extremely small initialization. Our work complements their results by allowing moderately small initializations, perturbations, and also considering stochastic gradient descent. See the more detailed comparison in \Cref{s:rip_result}.

The main idea in our proofs is the reduction to perturbed gradient flow (see \Cref{s:pert_flow}). To the best of our knowledge, the idea of reducing gradient descent to a perturbed gradient flow is entirely novel. However, there are similarities to backward error analysis, which is also based on the idea of constructing a continuous differential equation which passes through the points given by a discrete iteration. While it is hard to find direct connections to our work, backward error analysis has previously been used in deep learning research. For example, for studying the implicit regularization imposed by gradient descent with large learning rates \citep{barrett2021implicit}.

\section{Perturbed gradient flow}\label{s:pert_flow}
Both of our main results, \Cref{rip_theorem,sgd_theorem}, are applications of \Cref{master_theorem}, which we state in this section. The idea is as follows: We may formulate the dynamics of perturbed gradient descent (\Cref{pert_gd}) as follows:
\begin{align}\label{eq:pert_desc2}
  \begin{pmatrix}U_{(k+1)\eta}\\V_{(k+1)\eta}\end{pmatrix} &= (I+\eta\tilde R_{k\eta})\begin{pmatrix}U_{k\eta}\\V_{k\eta}\end{pmatrix}.
\end{align}
In the simplified case of matrix factorization (perfect information $\A = I$ and no noise $\tilde E = 0$), we have $\tilde R_t = R_t$ where $R_t \coloneqq \begin{pmatrix}0 & Y-U_tV_t\T \\ Y\T-V_tU_t\T & 0\end{pmatrix}$. The problem can be further simplified by considering gradient flow ($\eta \to 0^+$) instead of gradient descent ($\eta > 0$). The key insight is that both of these simplifications can be achieved at the cost of a perturbation term $E$. Specifically, we can choose $E$ such that \eqref{eq:pert_desc2} is perfectly interpolated by
\begin{align}\label{eq:pert_flow}
  \dt W_t = (R_t+E_t) W_t,
\end{align}
where $W_t = \begin{pmatrix}U_t\\V_t\end{pmatrix}$. Hence, bounds on $W_t$ from \eqref{eq:pert_flow} directly translate into bounds on the perturbed gradient descent iterates. Moreover, in the settings of \Cref{rip_theorem,sgd_theorem}, we have tight bounds on $E$ of the following form:
\begin{assumption}\label{asmp:pert}
  Assume $E$ is a piecewise smooth function of $t$. Let the constants $\eta, \gamma, \beta, \mu, \nu \ge 0$ be given. Moreover, assume the following holds for all $0 \le t \le T$:

  Let $\tau = t$ if $\eta = 0$ or $\tau = \lfloor t / \eta \rfloor \eta$ if $\eta > 0$, where $\lfloor \cdot \rfloor$ is the floor function. If $\norm{W_\tau} \le \frac{3}{2}\sqrt{\norm{Y}}$ and $\gamma \sigma_{r+1}^2(W_\tau) \le \norm{Y}$, then
  \begin{align*}
    \norm{E_t} &\le \beta \left(\norm{R_\tau}+\gamma \sigma_{r+1}^2(W_\tau)\right) + \mu\norm{Y}\\
    \norm{J E_t + E_t\T J} &\le \nu \norm{Y},
  \end{align*}
  where $J = \begin{pmatrix}I_{n_1} & 0\\0 & -I_{n_2}\end{pmatrix}$.
\end{assumption}
\begin{remark}
  Allowing for $\eta > 0$ in \Cref{asmp:pert} makes it easier to verify it for discrete iterations based on perturbed gradient descent. Clearly, $\norm{J E_t + E_t\T J} \le 2\norm{E_t}$. However, bounding through $\norm{E_t}$ leads to slightly worse sample complexities.
\end{remark}

We are ready to state our general result for perturbed gradient flow.
\begin{theorem}\label{master_theorem}
  Let the target matrix $Y \in \R^{n_1\times n_2}$ have $\mathrm{rank}(Y) = r \ge 1$ and condition number $\kappa = \frac{\norm{Y}}{\sigma_r(Y)}$. Let $W \colon \R_{\ge0} \to \R^{(n_1+n_2)\times h}$ be a continuous and piecewise smooth function which follows the perturbed gradient flow dynamics \eqref{eq:pert_flow} from time $0$ to time $T = \frac{5}{\sigma_r(Y)}\log\left(\frac{\sigma_r(Y)}{\norm{W_0}^2}\right)$. Here $\norm{W_0}$ is the norm of the initialization.

  Select a constant $0 < \theta \le \frac{1}{2}$, which will appear as an exponent in the final bound. Let $\alpha \ge 1$ be the misalignment at initialization (\Cref{def:pert_init}, see \Cref{remark:init}).

  Next, assume the perturbations $E$ in \eqref{eq:pert_flow} satisfy \Cref{asmp:pert} with 
  \[\gamma \ge 1,\ \beta \le \frac{c_2\theta^2}{\alpha\kappa^2}, \ \mu \le \frac{c_3}{\alpha\kappa^2},\ \nu \le \frac{c_4\theta^2}{\alpha\kappa^2 T \norm{Y}}\text{ and }\eta \le \frac{c_5}{\sigma_r(Y)}.\]

  Furthermore, assume the norm of the initialization satisfies \[\norm{W_0} \le \min\left(\frac{\sqrt{\kappa}}{\sqrt\alpha\gamma}, \frac{c_1\theta}{\alpha\kappa}\right)\sqrt{\norm{Y}}.\]

  Then the final reconstruction error $R_T$ satisfies 
  \begin{align}\label{master_bound}
    \norm{R_T} \le C_1\left(\beta\kappa\gamma+\sqrt{\kappa}\right)\left(\frac{\norm{W_0}^2}{\sigma_r(Y)}\right)^{1-\theta}\sigma_r(Y) + C_2\kappa \mu\norm{Y}.
  \end{align}
  The constants $c_1,c_2,c_3,c_4,c_5,C_1,C_2 > 0$ are universal.
\end{theorem}

\textbf{Simplifying the bound:} The bound $\norm{Y-U_{K\eta} V_{K\eta}\T} \lesssim \left(\frac{\epsilon^2}{\sigma_r(Y)}\right)^\frac{2}{3} \sigma_r(Y) + \kappa \xi$ presented in \Cref{rip_theorem,sgd_theorem} is a simplified version of \eqref{master_bound}. To see that the left-hand sides are the same, note that $K\eta = T$ and $\norm{R_T} = \norm{Y-U_TV_T\T}$. For the right-hand sides, first note that $\epsilon \coloneqq \norm{W_0}$. Select $\theta = 1/12$ and $\epsilon \le \frac{c\sqrt{\sigma_r(Y)}}{\gamma^2+\kappa}$. Hence, $\left(\beta\kappa\gamma+\sqrt{\kappa}\right)\left(\frac{\norm{W_0}^2}{\sigma_r(Y)}\right)^{1-\theta} \lesssim \left(\frac{\epsilon^2}{\sigma_r(Y)}\right)^\frac{2}{3} \sigma_r(Y)$. Finally, in the relevant settings, we have $\xi \lesssim \mu\norm{Y}$. See \Cref{s:rip_proof} for details. The exponent $\frac{2}{3}$ can be made arbitrarily close to 1 by making $\theta$ and $\epsilon$ small enough.

\section{Proofs of \texorpdfstring{\Cref{rip_theorem,sgd_theorem}}{}}\label{s:reduction}
We prove \Cref{rip_theorem,sgd_theorem} by reducing them to instances of \Cref{master_theorem}, which is proved in \Cref{s:master_proof_overview}. The detailed proofs of \Cref{rip_theorem,sgd_theorem} are given in \Cref{s:rip_proof,s:sgd_proof}, but we give an overview of the key ideas in this section.

Consider iteration $k$ of perturbed gradient descent (\Cref{pert_gd}), which may be formulated as
\begin{align}\label{eq:pert_gd_reduction}
  W_{(k+1)\eta} = (I+\eta \tilde R_{k\eta}) W_{k\eta},
\end{align}
where $W_t = \begin{pmatrix}U_t\\V_t\end{pmatrix}$, $\tilde R_{k\eta} = \begin{pmatrix} 0 & (\A_k^*\A_k)(\bar R)\\(\A_k^*\A_k)(\bar R)\T & 0\end{pmatrix} + \tilde E_k$ and $\bar R = Y-UV\T$.

Our goal is to find a perturbation $E$ such that the solution to $\dt W_t = (R_t+E_t)W_t$ interpolates \eqref{eq:pert_gd_reduction}, while satisfying \Cref{asmp:pert}. Here $R_t = \begin{pmatrix}0 & \bar R\\\bar R\T & 0\end{pmatrix}$. We show that the following choice of $E$ has the desired properties:

\[E_t = \frac{1}{\eta}\log\left(I+\eta\tilde R_{k\eta}\right) - R_t,\quad\text{ for }t \in [k\eta,(k+1)\eta).\]
The matrix log is taken as $\log(X) \coloneqq \sum_{k=1}^\infty \frac{(-1)^{k+1}}{k}(X-I)^k$, defined for $\norm{X-I} < 1$.

To bound $E_t$, split $\tilde R_{k\eta} = R_{k\eta} + \hat E_k^\A + \tilde E_k$, where $\hat E_k^\A \coloneqq \begin{pmatrix}0 & E^\A_k\\(E^\A_k)\T & 0\end{pmatrix}$ with $E^\A \coloneqq (\A^*\A)(\bar R)-\bar R$. Intuitively, $\norm{E_k^\A}$ represents the error stemming from imperfect measurements. We then have the following result (\Cref{prop:discrete} from \Cref{s:discrete}):

If $\eta \le 1/(12\norm{Y})$, $\norm{W_{k\eta}} \le \frac{3}{2}\sqrt{\norm{Y}}$ and $\norm{E^\A_k}+\norm{\tilde E_k} \le \norm{Y}$, we have the following bounds on $E_t$ 
\begin{align*}
  \norm{E_t} &\lesssim \eta\norm{Y} \norm{R_{k\eta}} + \norm{E^\A_k} + \norm{\tilde E_k},\\
  \norm{J E_t + E_t\T J} &\lesssim \eta\norm{Y}^2 + \norm{\tilde E_k}.
\end{align*}

This bound essentially verifies \Cref{asmp:pert}, if we can bound $\norm{E^\A_k}$. We bound this term differently for \Cref{rip_theorem,sgd_theorem}.

For \Cref{rip_theorem}, when the measurement operator satisfies RIP (\Cref{def:rip}) with rank $r+1$ and constant $\rho$, we use the following bound (\Cref{EA_bound} from \Cref{s:rip_proof_appendix}):
\begin{align}\label{top_EA_rip}
  \norm{E^\A} \lesssim \sqrt{r}\rho\left(\norm{R} + \frac{\min(n_1,n_2)}{r}\sigma_{r+1}^2(W)\right).
\end{align}
The factor $\sqrt r$ leads to a suboptimal factor $r$ in the sample complexity $m = O(r^2(n_1+n_2)\kappa^4)$. We believe this suboptimal factor $r$ is a fundamental limitation of RIP in our setting, since it is also present in all previous works, even in the simpler setting of symmetric PSD sensing. By directly analyzing stochastic gradient descent, we bypass the problem, as shown below.

For \Cref{sgd_theorem}, we assume measurement operators of the form $[\A(X)]_i = \frac{1}{\sqrt m} \langle X, A_i\rangle$, for $m$ measurements $i = 1,\dots,m$, where each $A_i \in \R^{n_1\times n_2}$ has i.i.d. entries from $\mathcal{N}(0,1)$. In this case, we have the bound (\Cref{EA_bound_sgd} from \Cref{s:sgd_proof_appendix}, informally): with high probability
\begin{align}\label{top_EA_sgd}
  \norm{E^\A} \lesssim \sqrt{\frac{r(n_1+n_2)}{m}}\left(\norm{R}+\sqrt{\frac{\min(n_1,n_2)}{r}}\sigma_{r+1}^2(W)\right).
\end{align}
This bound leads to the sample complexity $m = \tilde O(r(n_1+n_2)\kappa^8)$, which for well-conditioned matrices ($\kappa \lesssim 1$) is optimal up to log factors.

\section{Proof overview for \texorpdfstring{\Cref{master_theorem}}{}}\label{s:master_proof_overview}
The proof is divided into two phases, called the warm-up phase, and the local convergence phase. Each phase has a list of bounds which are maintained through real induction \citep{clark2019instructor}.
\subsection{Setup}\label{s:top_setup}
Recall the $(n_1+n_2)\times h$ matrix $W_t = \begin{pmatrix}U_t\\V_t\end{pmatrix}$. Also, $\hat Y = \begin{pmatrix} 0 & Y \\ Y\T & 0\end{pmatrix}$, $J = \begin{pmatrix} I_{n_1} & 0 \\ 0 & -I_{n_2} \end{pmatrix}$ and $R = \hat Y - \frac{1}{2}\left(WW\T-JWW\T J\right)$. We will bound several functions of $W$. For this purpose, we use projections. The matrix $\hat Y$ has $r$ positive and $r$ negative eigenvalues. Let $P$ be a unitary matrix which diagonalizes $\hat Y$ with decreasing diagonal. Let $\PA$,$P_0$,$P_-$ and $\PP$ consist of rows of $P$ corresponding to positive, zero, negative and non-positive eigenvalues, respectively. The structure of $\hat Y$ lets us pick $P_- \coloneqq \PA J$. We view $\PA W$,$P_0 W$,$\PA J W$ and $\PP W$ as parts of $W$ corresponding to the various parts of the target matrix $Y$. Intuitively, we would like $\PA W$ to fit $Y$, while the rest of $W$, that is, $\PP W$ should be as small as possible.

An alternative way of defining $\PA$ is as follows: Let $Y = U_Y \Lambda_Y V_Y\T$ be a compact singular value decomposition of $Y$. Then $\PA$ is the $r\times (n_1+n_2)$ matrix such that \[\PA \begin{pmatrix}U\\V\end{pmatrix} = \frac{1}{\sqrt 2}\left(U_Y\T U + V_Y\T V\right).\]
This lets us write the initial misalignment (\Cref{def:pert_init}) as $\alpha = \norm{W_0}/\sigma_r(\PA W_0)$.

Following \citet{rip_small_init} and \citet{rip_asymmetric}, we also split $W$ into a signal part or rank $r$ and a nuisance part
\[W_t = W_tQ_t + W_t(I-Q_t).\]
Here the projection $Q_t \coloneqq W_t\T \PA\T \left(\PA W_t W_t\T \PA\T\right)^{-1} \PA W_t$ extracts the signal part of $W_t$.

Finally, in the statement of \Cref{master_theorem} there was a small constant $0 < \theta \le 1/2$ which controlled the final exponent. In the proofs, it is more convenient to work with the small constant \[\delta = \frac{\theta}{32\sqrt\alpha\kappa} \le \frac{1}{64\sqrt\alpha\kappa}.\]

\subsection{Warm-up phase}
The main result of the warm-up phase is the following. The detailed proof is given in \Cref{s:warmup}.
\begin{theorem}\label{top_warmup}
  For all times $t \in [0,T_2]$, we have
  \begin{multicols}{2}
    \begin{enumerate}
      \item $\norm{W_t} \le \frac{3}{2}\sqrt{\norm{Y}}$\label{top_W}
      \item $\norm{\Bt t} \le \left(1+\frac{5t}{T_2}\right)\delta^2\norm{Y}$\label{top_imb}
      \item $\norm{\PA JW_t} \le \frac{\delta}{3\sqrt\alpha}\sqrt{\norm{Y}}$\label{top_PJW}
      \item $\norm{\PN W_t} \le \delta\sqrt{8\norm{Y}}$\label{top_NW}
      \item $\PP \left(\hat Y + \frac{1}{2}JW_tW_t\T J\right) \PP\T \preccurlyeq 2\delta\norm{Y}I$\label{top_PXP}
      \item $\norm{\PP W_t (\PA W_t)^\dagger} \le \alpha$\label{top_F}
      \item $\norm{W_t(I-Q_t)} \le \norm{W_0} e^{3\sqrt\alpha \delta \norm{Y}t}$\label{top_tW}
      \item $\sigma_r(\PA W_t) \ge \min\left(\sqrt{\frac{\norm{Y}}{\kappa}},\ \frac{\norm{W_0}}{\alpha} e^{\frac{2 \norm{Y}}{5\kappa}t}\right)$\label{top_A}.
    \end{enumerate}
  \end{multicols}
\end{theorem}
Here $^\dagger$ denotes the pseudoinverse.

The first bound $\norm{W_t} \le \frac{3}{2}\sqrt{\norm{Y}}$ shows that $\norm{W_t}$ does not become much larger than what is needed to fit $Y$. \Cref{top_imb} implies $\norm{\Bt t} \le 6\delta^2\norm{Y}$. Handling this "imbalance" $\norm{\Bt t} = \norm{U_t\T U_t-V_t\T V_t}$ is a key point in asymmetric matrix factorization and sensing \citep{balanced}. The nuisance term $\norm{W_t(I-Q_t)}$ (\Cref{top_tW}) grows exponentially from its small initialization. However, the rate of growth is very slow, so it is still very small at time $T_2$, when reconstruction is finished.

\Cref{top_F}, $\norm{\PP W_t (\PA W_t)^\dagger} \le \alpha$, bounds the "misalignment" of the column space of the signal part $W_tQ_t$. More specifically, $\norm{\PP W_t (\PA W_t)^\dagger} = \tan(\theta)$, where $\theta$ is the maximum angle from a vector in the column space of $W_tQ_t$ to the row space of $\PA$.

The bound $\norm{W_t} \le \frac{3}{2}\sqrt{\norm{Y}}$ from \Cref{top_W} is rather crude. To get a good sample complexity\footnote{We suspect that our results are sharper than \citet{rip_asymmetric} primarily because of these bounds. In their notation, they work with $\norm{L^T_{X,\perp}P_{Z_tQ_t}} \le \frac{c}{\kappa}$, which is essentially $\norm{\PP W_t (\PA W_t)^\dagger} \lesssim \delta$ in our notation. This requires $\kappa^2$ times more samples than our weaker bound $\norm{\PP W_t (\PA W_t)^\dagger} \le \alpha$. However, it lets them use bounds such as $\norm{\PP W_t Q_t} \le \norm{\PP W_t (\PA W_t)^\dagger}\norm{W_t} \lesssim \delta \sqrt{\norm{Y}}$ in place of the bounds in \Cref{top_PJW,top_NW,top_PXP}.}, we need sharper bounds for each of the three parts of $W_t$, that is, $\PA W_t$, $\PN W_t$ and $\PA JW_t$. The bounds $\norm{\PN W_t} \lesssim \delta\sqrt{\norm{Y}}$ and $\norm{\PA J W_t} \lesssim \delta\sqrt{\norm{Y}}$ in \Cref{top_PJW,top_NW} are straightforward. However, there is no such bound for $\PA W_t$, since this part grows big. Therefore, we instead show \[\PA W_t W_t\T \PA\T \preccurlyeq 2\PA \hat Y\PA\T + 4\delta \norm{Y} I,\] which is implied by \Cref{top_PXP}. The slightly stronger form in \Cref{top_PXP} is easier to work with in the proofs.

\Cref{top_A} shows that the size of $W$ grows exponentially in the warmup phase. The time $T_1$ is chosen such that $\sigma_r(\PA W_{T_1}) \ge \sqrt{\norm{Y}/\kappa}$ by \Cref{top_A}. This shows that the aligned part of $W$ is large enough to start the local convergence phase (\Cref{s:local_top}), where the reconstruction error quickly decreases.

\textbf{Proof sketch for \Cref{top_warmup}:} The general strategy is to bound the growth of each of the 8 items. For example, consider \Cref{top_NW}, $\norm{\PN W_t} \le \delta\sqrt{8\norm{Y}}$. First, we differentiate $\dt \left(\PN W_t\right) = \PN (R_t+E_t)W_t$. We then use that to give a bound of the following form (\Cref{prop:NW_bound} from \Cref{s:warmup}, informal):
\begin{align}\label{eq:top_NW}
  \text{If }\norm{\PN W_t} = \delta\sqrt{8\norm{Y}}\text{, then }\dt \norm{\PN W_t} < 0.
\end{align}

Note that $\dt \norm{\PN W_t}$ might technically not exist everywhere. In the detailed proofs in \Cref{s:warmup}, we formally deal with this problem using technical tools developed in \Cref{s:tech}.

Using 8 bounds of the form \eqref{eq:top_NW}, one for each item, we prove \Cref{top_warmup} with real induction (\Cref{real_ind} from \Cref{s:warmup}).

\subsection{Local convergence phase}\label{s:local_top}
For the local convergence phase, the main result is the following.
\begin{theorem}\label{top_local}
  For all times $t \in [T_1,T_2]$, we have
  \begin{multicols}{2}
    \begin{enumerate}
      \item $\norm{R_t} \le M^R_t$\label{top_R_bound}
      \item $\norm{\PN W_tQ_t} \le \frac{2}{5}M^R_t / \sqrt{\norm{Y}}$
    \end{enumerate}
  \end{multicols}
  Where
  \begin{align*}
    M^R_t &= \max\left(3\norm{Y}\exp\left(-\frac{2\norm{Y}}{5\kappa}(t-T_1)\right), M^R_\infty\right),\\
    M^R_\infty &= 64\left(\beta\gamma\kappa+\sqrt{\kappa}\right)\norm{W_0}^2\exp(6\sqrt\alpha \delta \norm{Y}T_2) + 10^3\mu\kappa \norm{Y}.
  \end{align*}
\end{theorem}

The detailed proof can be found in \Cref{s:local}. It is proved using a similar strategy to the warm-up phase, \Cref{top_warmup}. 

\Cref{top_R_bound}, $\norm{R_t} \le M^R_t$, shows that the reconstruction error $\norm{R_t}$ decreases exponentially until $\norm{R_t} \le M^R_\infty$. The time $T_2$ is chosen such that $\norm{R_t} \le M^R_\infty$, which implies the bound in \Cref{master_theorem}. 

It is unfortunately not clear how to directly prove $\norm{R_t} \le M^R_t$ by real induction. The problem is that there exists $W_t$ with $\norm{\PN W_t Q_t}$ relatively large, but with small $\norm{R_t}$. This makes direct induction on $\norm{R_t} \le M^R_t$ difficult. To this end, we add the bound $\norm{\PN W_tQ_t} \le \frac{2}{5}M^R_t / \sqrt{\norm{Y}}$, which makes sure the offending part of $W_t$ stays small.

\section{Acknowledgments}
The author would like to thank Ali Ramezani-Kebrya, Dominik Stöger, Vegard Antun, and Åsmund Hausken Sande for their helpful feedback and discussions about the paper.

\bibliographystyle{plainnat}
\bibliography{robust_sensing_updated}

\newpage
\appendix
\crefalias{section}{appendix} 

\section{Additional notations}
We let $\lambda_k(X)$ denote the $k$th largest eigenvalue of the symmetric matrix $X$. We call $u,v$ a top singular pair of a matrix $X$ if $\norm{u} = \norm{v} = 1$, $Xv = \norm{X}u$ and $X\T u = \norm{X}v$. A bottom singular pair is similarly defined with the smallest singular value $\sigma_{\rank(X)}(X)$ in place of $\norm{X}$. We differentiate some variables which are only piecewise smooth; the derivative then denotes the right derivative. We sometimes denote the time derivative of a variable $X$ by $\dot X \coloneqq \frac{d}{dt}X$. The notation $X^\dagger$ denotes the pseudoinverse of $X$. We use $\lfloor \cdot \rfloor$ and $\lceil \cdot \rceil$ denote the floor and ceiling functions. We denote the $k\times k$ identity matrix by $I_k$ or just $I$ if the shape is clear from context. Moreover, we write $0$ for the zero matrix. The following matrix logarithm is chosen for matrices $X$ satisfying $\norm{X-I} < 1$: $\log(X) = \sum_{k=1}^\infty \frac{(-1)^{k+1}}{k}(X-I)^k$.

\section{Proof of \texorpdfstring{\Cref{master_theorem}}{}}\label{s:master_proof}
We repeat the main assumptions of \Cref{master_theorem} in a way which is easy to reference later in the proof.

Let the target matrix $Y \in \R^{n_1\times n_2}$ have $\mathrm{rank}(Y) = r \ge 1$ and condition number $\kappa = \frac{\norm{Y}}{\sigma_r(Y)}$. As explained in \citet{du_asymmetric}\footnote{In addition to the arguments in \citet{du_asymmetric}, we also have an incomplete measurement operator $\A$. However, the RIP property is rotationally invariant, so this is not a problem.}, we can assume without loss of generality that $Y$ is a diagonal matrix with non-negative and decreasing diagonal. Since $\rank(Y) = r$, the first $r$ elements on the diagonal are at least $Y_{rr} = \sigma_r(Y) > 0$, and the rest are zero. In particular, this means $Y_{rr} = \sigma_r(Y)$.

Let $W \colon \R_{\ge0} \to \R^{(n_1+n_2)\times h}$ be continuous, piecewise smooth, and satisfy the differential equation
\begin{align}\label{eq:pert_flow2}
  \dt W_t = (R_t+E_t) W_t,
\end{align}
where $R_t = \hat Y - \frac{1}{2}\left(W_tW_t\T-JW_tW_t\T J\right)$ with $\hat Y = \begin{pmatrix}0 & Y\\Y\T & 0\end{pmatrix}$ and $J = \begin{pmatrix}I_{n_1} & 0\\0 & -I_{n_2}\end{pmatrix}$.

In \Cref{master_theorem}, there was an exponent $\theta$. It is more convenient in the proofs to instead use the constant $\delta = \frac{\theta}{32\sqrt\alpha\kappa}$, so
\begin{align}
  0 < \delta \le \frac{1}{64\sqrt\alpha\kappa}\label{delta_bound}.
\end{align}
Furthermore, we assume the following bounds:

\pagebreak
\begin{multicols}{2}
  \begin{align}
    \alpha,\gamma &\ge 1\label{alpha_gamma_bound}\\
    \norm{W_0} &\le \min\left(\frac{\sqrt{\kappa}}{\sqrt\alpha\gamma}, \frac{\delta}{3\sqrt\alpha}\right)\sqrt{\norm{Y}}\label{eps_bound}\\
    T_1 &= \frac{5}{4Y_{rr}}\log\left(\alpha^2\frac{Y_{rr}}{\norm{W_0}^2}\right)\label{T1_bound}\\
    T_2 &= \frac{5}{Y_{rr}}\log\left(\frac{Y_{rr}}{\norm{W_0}^2}\right)\label{T2_bound}
  \end{align}

  \begin{align}
    \beta &\le \frac{\delta^2}{13}\label{beta_bound}\\
    \mu &\le \frac{\delta^2}{4}\nonumber\\
    \nu &\le \frac{2\delta^2}{T_2\norm{Y}}\nonumber\\
    \eta &\le \frac{1}{Y_{rr}}\label{eta_bound}
  \end{align}
\end{multicols}

Moreover, assume that $E \colon \R_{\ge0} \to \R^{(n_1+n_2)\times (n_1+n_2)}$ is piecewise smooth and satisfies the following property for all $0 \le t \le T$:

Let \[\tau = \begin{cases}\lfloor t / \eta \rfloor \eta&\text{ if }\eta > 0\\t&\text{ if }\eta = 0\end{cases}.\]

If $\norm{W_\tau} \le \frac{3}{2}\sqrt{\norm{Y}}$ and $\gamma \sigma_{r+1}^2(W_\tau) \le \norm{Y}$, then
\begin{align}
  \norm{E_t} &\le \beta \left(\norm{R_\tau}+\gamma \sigma_{r+1}^2(W_\tau)\right) + \mu\norm{Y} \nonumber\\
             &\le \frac{\delta^2}{13} \left(\norm{R_\tau}+\gamma \sigma_{r+1}^2(W_\tau)\right) + \frac{\delta^2}{4}\norm{Y}\label{new_general_E_bound}\\
  \norm{J E_t + E_t\T J} &\le \nu \norm{Y} \le \frac{2\delta^2}{T_2}\label{JE_bound},
\end{align}

Under these assumptions, we will prove
\begin{align}\label{goal}
  \norm{R_{T_2}} \le 64\left(\beta\kappa\gamma+\sqrt{\kappa}\right)\norm{W_0}^2\left(\frac{Y_{rr}}{\norm{W_0}^2}\right)^{32\sqrt\alpha\kappa\delta} + 10^3 \kappa \mu\norm{Y}.
\end{align}

\subsection{Setup}\label{s:setup}
We first repeat some key definitions from \Cref{s:top_setup}, and then give names to some commonly occuring expressions below. Recall $W = \begin{pmatrix}U \\ V\end{pmatrix}$, $\hat Y = \begin{pmatrix} 0 & Y \\ Y\T & 0\end{pmatrix}$, $J = \begin{pmatrix} I_{n_1} & 0 \\ 0 & -I_{n_2} \end{pmatrix}$ and $R = \hat Y - \frac{1}{2}\left(WW\T-JWW\T J\right)$. We will need to extract various parts of the parameters $W$. We do this using projections. Let $\PA \in \R^{r \times (n_1+n_2)}$ and $\PN \in \R^{(n_1+n_2-2r)\times (n_1+n_2)}$ be such that $P = \begin{pmatrix}\PA \\ \PN \\ \PA J\end{pmatrix}$ is unitary and diagonalizes $\hat Y$. Specifically, $P \hat Y P\T$ is diagonal with decreasing diagonal. Equivalently, let $\PA$ be the unique partial isometry such that $\PA W = \frac{1}{\sqrt 2}(U_{1:r,:}+V_{1:r,:})$, where $U_{1:r,:} \in R^{r\times h}$ denotes rows $1$ through $r$ of $U$. Furthermore, define $\PN$ by $\PN W = \begin{pmatrix}U_{r+1:m} \\ V_{r+1:n}\end{pmatrix}$. Finally, let $\PP = \begin{pmatrix}\PN \\ \PA J\end{pmatrix}$, so $\PP W = \begin{pmatrix}\PN W \\ \frac{1}{\sqrt 2}(U_{1:r,:}-V_{1:r,:})\end{pmatrix}$. Note that $\PP$ extracts the complement of $\PA$, that is, $\PA\PA\T = I-\PP\PP\T$. Furthermore, note $\lambda_r(\PA \hat Y \PA) = Y_{rr}$, $\PN \hat Y = 0$ and $J\hat Y J = -\hat Y$ (\Cref{tiny_dilation}). More consequences of the setup are listed in \Cref{s:conseq}.

The part of $W$ aligned with $\hat Y$, we will call \[A \coloneqq \PA W.\] We note the pseudoinverse $A^\dagger \coloneqq A\T(AA\T)^{-1}$. This lets us define the projection onto the row space of $A$, which we call \[Q \coloneqq A^\dagger A.\] We use $Q$ to split $W$ into a signal part and a nuisance part $W = WQ + \tilde W$, where \[\tilde W \coloneqq W-WQ.\] We will show that $\norm{\tilde W}$ grows slowly from its small initialization, while the signal part $WQ$ has $\rank(WQ) = r$ and quickly minimizes the objective. We finally define \[F \coloneqq \PP W A^\dagger,\] which represents the misalignment of $WQ$ with respect to the target $\hat Y$.

\subsection{Proofs for the warm-up phase}\label{s:warmup}
In this section, we show that the bounds listed below all hold until convergence at time $T_2$.
\begin{definition}\label{ind:def}
  Let $S^{warmup}$ be the set of times $t \in [0,T_2]$ such that
  \begin{multicols}{2}
    \begin{enumerate}
      \item $\norm{W_t} \le \frac{3}{2}\sqrt{\norm{Y}}$\label{ind:W}
      \item $\norm{\Bt t} \le \left(1+\frac{5t}{T_2}\right)\delta^2\norm{Y}$\label{ind:B}
      \item $\norm{\PA JW_t} \le \frac{\delta}{3\sqrt\alpha}\sqrt{\norm{Y}}$\label{ind:PJW}
      \item $\norm{\PN W_t} \le \delta\sqrt{8\norm{Y}}$\label{ind:NW}
      \item $\PP \left(\hat Y + \frac{1}{2}JW_tW_t\T J\right) \PP\T \preccurlyeq 2\delta\norm{Y}I$\label{ind:PPX}
      \item $\norm{F_t} \le \alpha$\label{ind:F}
      \item $\norm{\tilde W_t} \le \norm{W_0} \exp\left(3\sqrt\alpha \delta \norm{Y}t\right)$\label{ind:tW}
      \item $\sigma_r(A_t) \ge \min\left(\sqrt{Y_{rr}},\ \frac{\norm{W_0}}{\alpha} e^{\frac{2 Y_{rr}}{5}t}\right)$\label{ind:A}.
    \end{enumerate}
  \end{multicols}
\end{definition}
\begin{theorem}\label{ind:thm}
  $S^{warmup} = [0,T_2]$. That is, the items in \Cref{ind:def} hold for all times $t \in [0,T_2]$.
\end{theorem}

We use the following formulation of real induction.
\begin{proposition}[Real induction]\label{real_ind}
  Assume the set $S \subset [a,b], a < b$ satisfies the following.
  \begin{enumerate}
    \item Base case: $a \in S$.
    \item Continuity: If $t \in (a,b]$, then $[a,t) \subset S \implies t \in S$.
    \item Induction step: If $t \in [a,b)$, then $[a,t] \subset S \implies [t,\hat t] \subset S$ for some $\hat t > t$.
  \end{enumerate}
  Then $S = [a,b]$.
\end{proposition}
\begin{proof}
  Assume for contradiction that $S \ne [a,b]$. Then we can find a greatest lower bound of the complement $x = \inf([a,b] \setminus S)$. The base case and induction step assumptions imply there exists $\hat t > a$ such that $[a,\hat t] \in S$, so $x \ge \hat t > a$. Next, we know $[a,x) \in S$, so continuity implies $x \in S$. If $x = b$, then $S = [a,b]$. Otherwise, an induction step on $[a,x] \in S$ yields a greater lower bound $\hat t > x$, contradicting the definition of $x$.
\end{proof}

Proving base case and continuity is straightforward.
\begin{proposition}[Base case]\label{ind:base}
  $0 \in S^{warmup}$.
\end{proposition}
\begin{proof}
  We have $\norm{W_0} \le \frac{\delta}{3\sqrt\alpha}\sqrt{\norm{Y}}$ by assumption \eqref{eps_bound} and $\delta \le 1$ by assumption \eqref{delta_bound}. Also, note that $\sigma_r(A_0) = \frac{\norm{W_0}}{\alpha}$ by \Cref{def:pert_init} in \Cref{s:main_results}. These directly imply \Cref{ind:W,ind:B,ind:PJW,ind:NW,ind:tW,ind:A} of \Cref{ind:def} for $t = 0$. For item \ref{ind:PPX}, since $\lambda_1(\PP \hat Y \PP\T) \le 0$, we have \[\lambda_1\left(\PP (\hat Y+\frac{1}{2}JW_0W_0\T J) \PP\T\right) \le \lambda_1(\PP \hat Y \PP\T) + \frac{1}{2}\norm{W_0}^2 < 2\delta\norm{Y}.\] Finally, by the definition of $F_0$, we have $\norm{F_0} = \norm{\PP W_0 A_0^\dagger} \le \norm{W_0}\norm{A_0^\dagger} = \frac{\norm{W_0}}{\sigma_r(A_0)} = \alpha$ for item \ref{ind:F}.
\end{proof}

\begin{proposition}[Continuity]\label{ind:cont}
  If $t > 0$ and $[0,t) \in S^{warmup}$, then $t \in S^{warmup}$.
\end{proposition}
\begin{proof}
  This follows from all expressions in \Cref{ind:def} being continuous with respect to $t$.
\end{proof}

\begin{proposition}[Induction step]
  Let $t \in [0,T_2)$ and assume 
  \begin{align}\label{ind:t}
    [0,t] \subset S^{warmup},
  \end{align}
  then there exists $\hat t > t$ such that $[t,\hat t] \subset S^{warmup}$.
\end{proposition}
\begin{proof}
  We pick $\hat t$ as the minimum of the ones provided in \Cref{cor:W,prop:B,cor:PJW,cor:NW,cor:PPX,cor:F,cor:tW,cor:A}.
\end{proof}

The rest of the subsection will be used to prove the induction steps for each item in \Cref{ind:def}, completing the proof above. Throughout, we fix a $t \in [0,T_2)$ such that the induction hypothesis \eqref{ind:t} holds. Below, we list some immediate consequences of this induction hypothesis.

\begin{lemma}\label{tW_bound}
  $\norm{\tilde W_t}^2 \le \min\left(\frac{1}{\gamma}, \delta\right)\frac{\norm{Y}}{\sqrt\alpha}$.
\end{lemma}
\begin{proof}
  By the induction hypothesis \eqref{ind:t}, \Cref{ind:tW} of \Cref{ind:def}, $\delta \le Y_{rr}/(64\sqrt\alpha\norm{Y})$ by \eqref{delta_bound}, and by the definition $T_2 = \frac{5}{Y_{rr}}\log\left(\frac{Y_{rr}}{\norm{W_0}^2}\right)$, we have $\norm{\tilde W_t} \le \norm{W_0}\exp\left(3\sqrt\alpha\delta\norm{Y}T_2\right) \le \norm{W_0}\left(\frac{Y_{rr}}{\norm{W_0}^2}\right)^\frac{1}{4}$. By assumption \eqref{eps_bound}, $\norm{W_0} \le \min\left(\frac{\sqrt{\kappa}}{\sqrt\alpha\gamma}, \frac{\delta}{3\sqrt\alpha}\right)\sqrt{\norm{Y}}$. In combination, these give the desired bound.
\end{proof}

\begin{lemma}\label{E_bound}
  $\norm{E_t} \le \frac{\delta^2}{2}\norm{Y}$.
\end{lemma}
\begin{proof}
  Recall the bound \eqref{new_general_E_bound}:
  \[\norm{E_t} \le \frac{\delta^2}{13}\left(R_\tau+\gamma \sigma_{r+1}^2(W_\tau)\right)+\frac{\delta^2}{4}\norm{Y},\] for some $\tau \in [0,t]$. Next, \Cref{tW_bound} gives $\norm{\tilde W_\tau}^2 \le \frac{\norm{Y}}{\gamma}$ which implies $\sigma_{r+1}^2(W_\tau) \le \frac{\norm{Y}}{\gamma}$ by \Cref{tiny_tWrW}. Additionally, $\norm{R_\tau} \le \frac{17}{8}\norm{Y}$ by \Cref{tiny_R_bound}. Inserting these into the bound on $\norm{E_t}$ yields
  \[\norm{E_t} \le \delta^2\left(\frac{1}{13}\left(\frac{17}{8}+1\right)+\frac{1}{4}\right)\norm{Y} \le \frac{\delta^2}{2}\norm{Y}.\]
\end{proof}

\begin{lemma}\label{EB_bound}
  $\norm{E_t} + \frac{1}{2}\norm{\Bt t} \le \frac{7\delta^2}{2}\norm{Y}$.
\end{lemma}
\begin{proof}
  $\norm{E_t} \le \frac{\delta^2}{2}\norm{Y}$ by \Cref{E_bound} and $\norm{\Bt t} \le 6\delta^2\norm{Y}$ by \Cref{ind:B} of the induction hypothesis.
\end{proof}

Now we prove the induction steps for \Cref{ind:W,ind:B,ind:PJW,ind:NW,ind:PPX} of \Cref{ind:def}.

\begin{proposition}
  If $\norm{W_t} = \frac{3}{2}\sqrt{\norm{Y}}$ and $u,v$ is a top singular pair of $W_t$. Then $u\T \frac{d}{dt} W_t v < 0$.
\end{proposition}
\begin{proof}
  \begin{align*}
    u\T \frac{d}{dt} W_t v &= u\T (R_t+E_t)W_t v\\
                           &= u\T \left(\hat Y + E_t - \frac{1}{2}(W_tW_t\T - JW_tW_t\T J)\right)W_t v\\
                           &\le (\norm{Y}+\norm{E_t}+\norm{\Bt t})\norm{W_t} - \frac{1}{2}u\T W_tW_t\T W_t v\\
                           &= (\norm{Y}+\norm{E_t} + \frac{1}{2}\norm{\Bt t})\norm{W_t} - \frac{1}{2}\norm{W_t}^3\\
                           &\le \left(\norm{Y} + \frac{7}{2}\delta^2\norm{Y}\right)\frac{3}{2}\sqrt{\norm{Y}} - \frac{27}{16}\norm{Y}^\frac{3}{2} < 0.
  \end{align*}
  The second identity expands $R_t$. Next, the inequality bounds products by products of norms, using $\norm{u} = \norm{v} = \norm{J} = 1$ and $\norm{Y} = \norm{\hat Y}$ (\Cref{tiny_dilation}). Then, we use that $u,v$ is a top singular pair of $W_t$. The penultimate inequality applies \Cref{EB_bound} and the assumption on $\norm{W_t}$. Finally, we use the assumption $\delta \le \frac{1}{64}$ from \eqref{delta_bound}.
\end{proof}

\begin{corollary}\label{cor:W}
  There exists $\hat t > t$ such that $\norm{W_\tau} \le \frac{3}{2}\sqrt{\norm{Y}}$ for all $\tau \in [t,\hat t]$.
\end{corollary}
\begin{proof}
  Apply \Cref{tech_norm} with $X \coloneqq W$ and $c = \frac{3}{2}\sqrt{\norm{Y}}$.
\end{proof}

\begin{proposition}\label{prop:B}
  There exists $\hat t > t$ such that $\norm{\Bt \tau} \le \delta^2\norm{Y}+\frac{5\tau}{T_2}\delta^2\norm{Y}$ for all $\tau \in [t,\hat t]$.
\end{proposition}
\begin{proof}
  We differentiate $\frac{d}{dt}(\Bt t) = W_t\T (R_t+E_t)\T J W_t + W_t\T J (R_t+E_t) W_t$. Noting that $R_t$ is symmetric, and $R_t J+JR_t = 0$ by \Cref{tiny_dilation}, we simplify $\frac{d}{dt}(\Bt t) = W_t\T (E_t\T J + JE_t) W_t$. Using $\norm{J} = 1$, $\norm{W_t} \le \frac{3}{2}\sqrt{\norm{Y}}$ by the induction hypothesis \eqref{ind:t} and $\norm{JE_t+E_t\T J} \le \frac{2\delta^2}{T_2}$ by \eqref{JE_bound}, this implies
  \begin{align}\label{dB_bound}
    \norm{\frac{d}{dt}(\Bt t)} &= \norm{W_t\T (E_t\T J + JE_t) W_t}
                              \le \norm{W_t}^2\norm{JE_t+E_t\T J}
                              \le \frac{9\delta^2}{2T_2}\norm{Y}.
  \end{align}
  Hence, we can find $\hat t > t$ small enough that for any $\tau \in [t,\hat t]$, we have
  \begin{align*}
    \norm{\Bt \tau} &\le \norm{\Bt t} + (\tau-t)\norm{\frac{d}{dt}(\Bt t)} + o(\tau-t)\\
                    &\le \left(1+\frac{5t}{T_2}\right)\delta^2\norm{Y}+(\tau-t)\frac{9\delta^2}{2T_2}\norm{Y} + o(\tau-t) \le \left(1+\frac{5\tau}{T_2}\right)\delta^2\norm{Y}.
  \end{align*}
  The first inequality is a property of the right derivative and the triangle inequality. The second inequality uses the induction hypothesis \eqref{ind:t} to bound $\norm{\Bt t}$, and \eqref{dB_bound}. The final inequality bounds the little-o when picking $\hat t$ small enough.
\end{proof}

\begin{proposition}
  Assume $\norm{\PA JW_t} = \frac{\delta}{3\sqrt\alpha}\sqrt{\norm{Y}}$ and $u,v$ is a top singular pair of $\PA JW_t$. Then $\frac{d}{dt}(u\T \PA J W_t v) < 0$.
\end{proposition}
\begin{proof}
  \begin{align}
    &\frac{d}{dt}(u\T \PA J W_t v) = u\T \PA J (R_t+E_t)W_t v\\
                                  &\quad= u\T \PA J \left(\hat Y + E_t - \frac{1}{2}(W_tW_t\T - JW_tW_t\T J)\right)W_t v \nonumber\\
                                  &\quad\le u\T \PA J\hat Y W_tv - \frac{1}{2}u\T \PA J W_t W_t\T W_t v + \norm{E_t}\norm{W_t}+\frac{1}{2}\norm{W_t}\norm{\Bt t}.\label{eq:two_terms}
  \end{align}
  The second equality expands the definition of $R_t$. The inequality bounds products by products of norms, using $\norm{u} = \norm{v} = \norm{J} = 1$.

  We will bound the first two terms as follows. For the second term, we can apply the property $u\T \PA J W_t = \norm{\PA J W_t} v\T$ of the top singular pair $u,v$. This implies 
  \[-\frac{1}{2}u\T \PA J W_t W_t\T W_t v = -\frac{1}{2}\norm{\PA J W_t} v\T W_t\T W_t v \le 0.\]
  For the first term, we have
  \begin{align*}
    u\T \PA J\hat Y W_tv &= -u\T \PA \hat Y J W_tv = -u\T \PA \hat Y \PA\T\PA J W_tv\\
                         &= -\norm{\PA J W_t} u\T \PA \hat Y \PA\T u \le -Y_{rr}\norm{\PA J W_t}.
  \end{align*}
  The first equality uses $J\hat Y = -\hat Y J$ by \Cref{tiny_dilation}. Next, $\PA \hat Y = \PA\hat Y \PA\T\PA$ by \Cref{tiny_PPYP}. The third equality uses $\PA J W_t v = \norm{\PA J W_t}u$ for the top singular pair $u,v$. The final inequality follows from $\lambda_r\left(\PA \hat Y \PA\T\right) = Y_{rr}$ by \Cref{tiny_PYP}.

  Applying the bounds on the first two terms of \eqref{eq:two_terms}, we hence have
  \begin{align*}
    \frac{d}{dt}(u\T \PA J W_t v) &\le -Y_{rr} \norm{\PA J W_t} + \left(\norm{E_t} + \frac{1}{2}\norm{\Bt t}\right)\norm{W_t}\\
                                  &\le -Y_{rr} \frac{\delta}{3\sqrt\alpha}\sqrt{\norm{Y}}+\frac{21}{4}\delta^2\norm{Y}^\frac{3}{2} < 0.
  \end{align*}
  The last two inequalities follow from $W_t \le \frac{3}{2}\sqrt{\norm{Y}}$ from the induction hypothesis \eqref{ind:t}, \Cref{EB_bound} and $\delta \le \frac{Y_{rr}}{64\sqrt\alpha\norm{Y}}$ by assumption \eqref{delta_bound}.
\end{proof}

\begin{corollary}\label{cor:PJW}
  There exists $\hat t > t$ such that $\norm{\PA J W_\tau} \le \frac{\delta}{3\sqrt\alpha}\sqrt{\norm{Y}}$ for all $\tau \in [t,\hat t]$.
\end{corollary}
\begin{proof}
  Apply \Cref{tech_norm} with $X \coloneqq \PA J W$ and $c = \frac{\delta}{3\sqrt\alpha}\sqrt{\norm{Y}}$.
\end{proof}

\begin{proposition}\label{prop:NW_bound}
  Assume $\norm{\PN W_t} = \delta\sqrt{8\norm{Y}}$. Let $u,v$ be a top singular pair of $\PN W_t$. Then $\frac{d}{dt}(u\T \PN W_t v) < 0$.
\end{proposition}
\begin{proof}
  \begin{align*}
    \frac{d}{dt}(u\T \PN W_t v) &= u\T \PN (R_t+E_t)W_t v \nonumber\\ 
                                &= u\T \PN \left(\hat Y + E_t - \frac{1}{2}(W_tW_t\T - JW_tW_t\T J)\right)W_t v \nonumber\\
                                &\le -\frac{1}{2}u\T \PN W_t W_t\T W_t v + \norm{E_t}\norm{W_t v}+\frac{1}{2}\norm{\PN J W_t}\norm{\Bt t}\\
                                &= -\frac{1}{2}\norm{\PN W_t} \norm{W_t v}^2 + \norm{E_t}\norm{W_t v}+\frac{1}{2}\norm{\PN W_t}\norm{\Bt t}\\
                                &\le \left(-\frac{1}{2}\norm{\PN W_t}^2+\norm{E_t}+\frac{1}{2}\norm{\Bt t}\right)\norm{W_t v} < 0.
  \end{align*}
  The second equality expands the definition of $R_t$. The first inequality bounds products by products of norms and uses $\PN \hat Y = 0$. The next equality uses the property $u\T \PN W_t = \norm{\PN W_t}v\T$ of the top singular pair $u,v$, along with $\norm{\PN J W_t} = \norm{\PN W_t}$ by \Cref{tiny_NJX}. The second inequality applies the inequality $\norm{W_tv} \ge \norm{\PN W_tv} = \norm{\PN W_t}$. The final inequality uses $\frac{1}{2}\norm{\PN W_t}^2 = 4\delta^2\norm{Y} > \norm{E_t}+\frac{1}{2}\norm{\Bt t}$ by the assumption on $\norm{\PN W_t}$ and \Cref{EB_bound}.
\end{proof}

\begin{corollary}\label{cor:NW}
  There exists $\hat t > t$ such that $\norm{\PN W_\tau} \le \delta\sqrt{8\norm{Y}}$ for all $\tau \in [t,\hat t]$.
\end{corollary}
\begin{proof}
  Apply \Cref{tech_norm} with $X \coloneqq \PN W_\tau$ and $c = \delta\sqrt{8\norm{Y}}$.
\end{proof}

\begin{proposition}
  Let $X_t \coloneqq \hat Y + \frac{1}{2}JW_tW_t\T J$. Assume \[\lambda_1(\PP X_t \PP\T) = 2\delta\norm{Y}.\] Let $v$ be an eigenvector to the largest eigenvalue of $\PP X_t \PP\T$. Then $\frac{d}{dt}\left(v\T \PP X_t \PP\T v\right) < 0$.
\end{proposition}
\begin{proof}
  \begin{align}
    &\frac{d}{dt}\left(v\T \PP X_t \PP\T v\right) = \frac{1}{2}v\T \PP J\left((R_t+E_t)W_t W_t\T+W_t W_t\T(R_t+E_t)\T\right)J\PP\T v \nonumber\\
    &\quad= v\T \PP J(R_t+E_t)W_t W_t\T J\PP\T v \nonumber\\
    &\quad= v\T \PP J\left(\hat Y - \frac{1}{2}(W_tW_t\T-JW_tW_t\T J) + E_t\right) W_t W_t\T J\PP\T v \nonumber\\
    &\quad\le v\T \PP J\left(\hat Y - \frac{1}{2}W_tW_t\T\right)W_t W_t\T J\PP v + \left(\norm{E_t}+\frac{1}{2}\norm{\Bt t}\right)\norm{W}^2.\label{eq:first_term}
  \end{align}
  For the first equality, differentiate $X_t$ by the product rule. The second equality uses the equality $x\T Z x = x\T Z\T x$ for vectors $x$ and square matrices $Z$. The third equality expands the definition of $R_t$. The final inequality bounds products by products of norms.

  We can simplify the first term as follows:
  \begin{align*}
    v\T \PP J\left(\hat Y - \frac{1}{2}W_tW_t\T\right)W_t W_t\T J\PP\T v
    &= v\T \PP \left(-\hat Y - \frac{1}{2}JW_tW_tJ\T\right)J W_t W_t\T J\PP\T v\\
    = 2 v\T \PP X_t(\hat Y-X_t)\PP\T v &= -2 \norm{X_t\PP\T v}^2 + 2v\T \PP X_t\hat Y\PP\T v.
  \end{align*}
  The first equality uses $J \hat Y = -\hat Y J$ by \Cref{tiny_dilation} and $J^2 = I$. Next, the definition $X_t \coloneqq \hat Y + \frac{1}{2}JW_tW_t\T J$ is used.

  The first term can be bounded by $\norm{\PP} = 1$ and $v$ being an eigenvector of $\PP X_t\PP\T$: \[-2\norm{X_t\PP\T v}^2 \le -2\norm{\PP X_t\PP\T v}^2 = -2\lambda_1^2(\PP X_t\PP\T).\]
  The other term is non-positive because
  \begin{align*}
    2v\T \PP X_t\hat Y\PP\T v &= 2v\T \PP X_t\PP\T \PP \hat Y\PP\T v = 2\lambda_1(\PP X_t\PP\T) v\T \PP \hat Y\PP\T v \le 0.
  \end{align*}
  The first equality follows from $\hat Y\PP = \PP\T\PP \hat Y \PP$ by \Cref{tiny_PPYP}. The second equality follows from $v$ being an eigenvector of $\PP X_t\PP\T$. The final inequality is a consequence of $\PP \hat Y\PP\T$ being negative semi-definite and $\lambda_1(\PP X_t\PP\T) = 2\delta\norm{Y} > 0$ by assumption.

  Inserting the bounds into \eqref{eq:first_term}, we get
  \begin{align*}
    \frac{d}{dt}\left(v\T \PP X_t \PP\T v\right) 
    &\le -2\lambda_1^2(\PP X_t\PP) + \left(\norm{E_t}+\frac{1}{2}\norm{\Bt t}\right)\norm{W}^2\\
    &\le -8\delta^2\norm{Y}^2 + \frac{63}{8}\delta^2\norm{Y}^2 < 0.
  \end{align*}
  The second inequality uses the assumption on $\lambda_1(\PP X_t\PP)$, $\norm{W_t} \le \frac{3}{2}\sqrt{\norm{Y}}$ by the induction hypothesis \eqref{ind:t}, and \Cref{EB_bound}.
\end{proof}

\begin{corollary}\label{cor:PPX}
  There exists $\hat t > t$ such that $\lambda_1\left(\PP (\hat Y + \frac{1}{2}JW_tW_t\T J) \PP\T\right) \le 2\delta\norm{Y}$ for all $\tau \in [t,\hat t]$.
\end{corollary}
\begin{proof}
  Apply \Cref{tech_eig} with $X \coloneqq \PP (\hat Y + \frac{1}{2}JW W \T J) \PP\T$ and $c = 2\delta\norm{Y}$.
\end{proof}

Getting nice expressions for the derivatives of $F$ and $\tilde W$ is non-trivial, so we present the detailed derivations below.

\begin{lemma}\label{dF}
  Let $X_t \coloneqq \hat Y + \frac{1}{2}JW_tW_t\T J$. Then
  \begin{align*}
    \frac{d}{dt}F_t &= (\PP-F_t\PA)(X_t + E_t)(\PA\T + \PP\T F_t)\\
                    &\quad+ \PP \tilde W_t \tilde W_t\T\left[(X_t + E_t) \PA\T (A_tA_t\T)^{-1}-\PP\T F_t\right].
  \end{align*}
\end{lemma}
\begin{proof}
  Recall $F_t = \PP W_t A_t^\dagger$, $Q_t = A_t^\dagger A_t$ and $A_t = \PA W_t$. We first differentiate $A_t^\dagger$ as follows
  \begin{align*}
    \frac{d}{dt}A_t^\dagger &= \frac{d}{dt}\left(A_t\T(A_tA_t\T)^{-1}\right)\\
                            &= \dot A_t\T(A_tA_t\T)^{-1} - A_t\T(A_tA_t\T)^{-1}\left(\dot A_t A_t\T+A_t \dot A_t\T\right)(A_tA_t\T)^{-1}\\
                            &= (I-Q_t) \dot A_t\T(A_tA_t\T)^{-1} - A_t^\dagger \dot A_t A_t^\dagger\\
                            &= (I-Q_t) \dot W_t\T \PA\T (A_tA_t\T)^{-1} - A_t^\dagger \PA \dot W_t A_t^\dagger.
  \end{align*}
  First, expand the pseudoinverse $^\dagger$. For the second equality, use the chain rule, product rule, and the derivative of the matrix inverse. Next, use the definition $Q_t = A_t^\dagger A_t$. Finally, apply $\dot A_t = \PA \dot W_t$.

  Now we differentiate $F_t$:
  \begin{align*}
    &\frac{d}{dt}F_t = \frac{d}{dt}(\PP W_t A_t^\dagger)\\
                  &= \PP \dot W_t A^\dagger + \PP W_t \left[(I-Q_t) \dot W_t\T \PA\T (A_tA_t\T)^{-1} - A_t^\dagger \PA \dot W_t A_t^\dagger\right]\\
                  &= (\PP-F_t\PA) \dot W_t A^\dagger + \PP W_t (I-Q_t)^2 \dot W_t\T \PA\T (A_tA_t\T)^{-1}\\
                  &= (\PP-F_t\PA) \left(X_t + E_t - \frac{1}{2}W_tW_t\T\right) W_t A^\dagger\\
                  &\quad+ \PP \tilde W_t \tilde W_t\T\left(X_t + E_t - \frac{1}{2}W_tW_t\T\right) \PA\T (A_tA_t\T)^{-1}\\
                  &= (\PP-F_t\PA) \left(X_t + E_t\right)(\PA\T + \PP\T F_t)
                   + \PP \tilde W_t \tilde W_t\T\left[(X_t + E_t) \PA\T (A_tA_t\T)^{-1}-\PP\T F_t\right].
  \end{align*}
  The second identity applies the product rule and inserts the expression for $\frac{d}{dt}A^\dagger$. Next, simplify using the definition of $F_t$ and $I-Q_t = (I-Q_t)^2$. For the fourth equality, expand $\dot W_t = (R_t+E_t)W_t = (X_t+E_t-\frac{1}{2}W_tW_t\T)W_t$ and use the definition $\tilde W_t = W_t(I-Q_t)$. The last equality uses $W_t A_t^\dagger = \PA\T+\PP\T F_t$ by \Cref{tiny_WAd} and simplifies
  \begin{align*}
    (\PP-F_t\PA)W_tW_t\T W_t A^\dagger = \PP\tilde W_t\tilde W_t\T W_t W_t\T \PA\T(A_tA_t\T)^{-1} = \PP \tilde W_t \tilde W_t\T \PP\T F_t.
  \end{align*}
  To see the first equality, note that the definitions $F_t = \PP W_t A^\dagger$, $A_t = \PA W_t$ and $Q_t = A_t^\dagger A_t$ imply $(\PP-F_t\PA)W_tW_t\T = (\PP-\PP W_t A^\dagger\PA)W_tW_t\T = \PP(W_t-W_t Q_t)W_t\T = \PP \tilde W_t\tilde W_t\T$. Furthermore, $W_t\T \PA\T(A_tA_t\T)^{-1} = A_t^\dagger$. The second equality applies $W_t A_t^\dagger = \PA\T+\PP\T F_t$ by \Cref{tiny_WAd} and $\tilde W_t\T \PA\T = 0$ by \Cref{tiny_PPtW}.
\end{proof}

\begin{lemma}\label{dtW}
  Let $X_t \coloneqq \hat Y + \frac{1}{2}JW_tW_t\T J$. Then
  \begin{align*}
    \frac{d}{dt} \tilde W_t &= \PP\T(\PP-F_t\PA) \left(X_t + E_t\right) \tilde W_t - \tilde W_t\tilde W_t\T\left[(X_t + E_t)\PA\T(A_t^\dagger)\T-\frac{1}{2}W_t+\tilde W_t\right].
  \end{align*}
\end{lemma}
\begin{proof}
  Recall $\tilde W_t = W_t (I-Q_t)$ and $Q_t = A_t^\dagger A_t$. By a calculation similar to calculating $\frac{d}{dt}A_t^\dagger$ in \Cref{dF}, we calculate $\frac{d}{dt}Q_t = (I-Q_t)\dot W_t\T\PA\T(A_t^\dagger)\T + A_t^\dagger\PA \dot W_t (I-Q_t)$. 
  \begin{align*}
    & \frac{d}{dt} \tilde W_t = \dot W_t (I-Q_t) - W_t \frac{d}{dt} Q_t\\
                            &= (I - W_t A_t^\dagger\PA) \dot W_t (I-Q_t) - W_t(I-Q_t)\dot W_t\T\PA\T(A_t^\dagger)\T\\
                            &= (I - W_t A_t^\dagger\PA) \left(X_t + E_t - \frac{1}{2}W_tW_t\T\right) \tilde W_t
                            - \tilde W_t\tilde W_t\T\left(X_t + E_t - \frac{1}{2}W_tW_t\T\right)\PA\T(A_t^\dagger)\T\\
                            &= \PP\T(\PP-F_t\PA) \left(X_t + E_t - \frac{1}{2}W_tW_t\T\right) \tilde W_t
                            - \tilde W_t\tilde W_t\T\left(X_t + E_t - \frac{1}{2}W_tW_t\T\right)\PA\T(A_t^\dagger)\T\\
                            &= \PP\T(\PP-F_t\PA) (X_t + E_t) \tilde W_t
                            - \tilde W_t\tilde W_t\T\left[(X_t + E_t)\PA\T(A_t^\dagger)\T-\frac{1}{2}W_t+\tilde W_t\right].
  \end{align*}
  The first identity uses the product rule on the definition $\tilde W_t = W_t(I-Q_t)$. Next, $\frac{d}{dt}Q_t$ is expanded. The third equality expands $\dot W_t = (X_t+E_t-\frac{1}{2}W_tW_t\T)W_t$, applies the definition of $\tilde W_t$, and simplifies $W_t(I-Q_t)W_t\T = W_t(I-Q_t)^2W_t\T = \tilde W_t \tilde W_t\T$. For the fourth equality, $W_t A_t^\dagger = \PA\T+\PP\T F_t$ by \Cref{tiny_WAd} and $I-\PA\T \PA = \PP\T\PP$, imply $I-W_tA_t^\dagger \PA = \PP\T(\PP-F_t\PA)$. The final equality simplifies
  \begin{align*}
    -\PP\T(\PP-F_t\PA)W_tW_t\T \tilde W_t + \tilde W_t\tilde W_t\T W_tW_t\T \PA\T(A_t^\dagger)\T = \tilde W_t\tilde W_t\T (W_t-2\tilde W_t).
  \end{align*}
  To see this, use $(\PP-F_t\PA)W_tW_t\T = \PP \tilde W_t\tilde W_t\T$ from the proof of \Cref{dF} and $\PP\T\PP \tilde W_t = \tilde W_t$ by \Cref{tiny_PPtW} to simplify the first term to $-\tilde W_t \tilde W_t\T\tilde W_t$. Next, simplify the second term by $W_t\T \PA\T(A_t^\dagger)\T = Q_t$. Finally, rewrite using $W_t Q_t = W_t-\tilde W_t$.
\end{proof}

We show three more immediate consequences of the induction hypothesis \eqref{ind:t} before we prove the induction step for the remaining items.

\begin{lemma}\label{tWAd_bound}
  $\norm{\tilde W_t}\norm{A_t^\dagger} \le \alpha$.
\end{lemma}
\begin{proof}
  By the induction hypothesis \eqref{ind:t}, we have $\sigma_r(A_t) \ge \min\left(\sqrt{Y_{rr}},\ \frac{\norm{W_0}}{\alpha} \exp\left(\frac{2Y_{rr}}{5}t\right)\right)$ and $\norm{\tilde W_t} \le \norm{W_0} \exp\left(3\sqrt\alpha \delta \norm{Y}t\right)$. Note $\norm{A_t^\dagger} = \frac{1}{\sigma_r(A_t)}$. If $\sigma_r(A_t) \ge \frac{\norm{W_0}}{\alpha} \exp\left(\frac{2Y_{rr}}{5}t\right)$, then $\norm{A_t^\dagger} \le \frac{\alpha}{\norm{W_0}} \exp\left(-\frac{2Y_{rr}}{5}t\right)$, so since $\delta \le \frac{Y_{rr}}{64\sqrt\alpha\norm{Y}}$ by assumption \eqref{delta_bound}, we have
  \begin{align*}
    \norm{\tilde W_t}\norm{A_t^\dagger} \le \alpha\exp\left((3\sqrt\alpha \delta \norm{Y}-2Y_{rr}/5)t\right) \le \alpha.
  \end{align*}
  For the other case, $\sigma_r(A_t) \ge \sqrt{Y_{rr}}$, use \Cref{tW_bound} to get $\norm{\tilde W_t}^2 \le \delta\norm{Y}$. Together, they give $\norm{\tilde W_t}\norm{A_t^\dagger} \le \sqrt{\frac{\delta\norm{Y}}{Y_{rr}}} \le \alpha$.
\end{proof}

\begin{lemma}\label{EJW_bound}
  Let $X_t \coloneqq \hat Y + \frac{1}{2}JW_tW_t\T J$. Then \[\norm{\PP X_t \PA\T} + \norm{E_t} \le \norm{E_t}+\frac{1}{2}\norm{\PA JW_t}\norm{W_t} \le \frac{\delta}{3\sqrt\alpha}\norm{Y}.\]
\end{lemma}
\begin{proof}
  For the first inequality, we bound
  \begin{align*}
    \norm{\PP X_t \PA\T} = \norm{\frac{1}{2}\PP JW_tW_t\T J\PA\T} \le \frac{1}{2}\norm{\PA JW_t}\norm{W_t}.
  \end{align*}
  The equality uses $\PP \hat Y \PA\T = 0$ by \Cref{tiny_PPYP}. The inequality bounds the product by a product of norms, using $\norm{J} = \norm{\PP} = 1$.
  
  To prove the second inequality in the statement, insert bounds on the norms from the induction hypothesis \eqref{ind:t} for \Cref{ind:W,ind:PJW} in \Cref{ind:def} and \Cref{E_bound}.
\end{proof}

\begin{lemma}\label{PPW_bound}
  $\norm{\PP W_t} \le 3\delta\sqrt{\norm{Y}}$.
\end{lemma}
\begin{proof}
  Since $\PP$ can be split into orthogonal parts $\PA J$ and $\PN$, the induction hypothesis \eqref{ind:t} gives $\norm{\PP W_t}^2 = \norm{\PA J W_t}^2 + \norm{\PN W_t}^2 < 9\delta^2\norm{Y}$.
\end{proof}

We now prove the induction step for the remaining \Cref{ind:F,ind:tW,ind:A}.

\begin{proposition}
  If $\norm{F_t} = \alpha$ and $u,v$ is a top singular pair of $F_t$. Then $\frac{d}{dt}(u\T F_t v) < 0$.
\end{proposition}
\begin{proof}
  Let $X_t \coloneqq \hat Y + \frac{1}{2}JW_tW_t\T J$. Then
  \begin{align*}
    \frac{d}{dt}(u\T F_t v) &= u\T (\PP-F_t\PA) \left(X_t + E_t\right)(\PA\T + \PP\T F_t)v\\
    &\quad\quad+ u\T \PP \tilde W_t \tilde W_t\T\left[(X_t + E_t) \PA\T (A_tA_t\T)^{-1}-\PP\T F_t\right]v\\
    &= (u\T\PP-\norm{F_t}v\T\PA) (X_t+E_t)(\PA\T v + \norm{F_t}\PP\T u)\\
    &\quad\quad + u\T \PP \tilde W_t \tilde W_t\T(X_t+E_t) \PA\T (A_tA_t\T)^{-1}v-\norm{F_t}\norm{\tilde W_t\T \PP\T u}^2\\
    &\le (u\T\PP-\norm{F_t}v\T\PA) (X_t+E_t)(\PA\T v + \norm{F_t}\PP\T u)\\
    &\quad\quad+ u\T \PP \tilde W_t \tilde W_t\T\PP\T \PP(X_t+E_t) \PA\T (A_tA_t\T)^{-1}v\\
    &\le -\norm{F}v\T\PA X_t \PA\T v + \norm{F}u\T\PP X_t \PP\T u\\
    &\quad\quad+ \left((1+\norm{F})^2+\norm{\tilde W_t}^2\norm{A^\dagger}^2\right)(\norm{\PP X_t \PA} + \norm{E_t})\\
    &\le -Y_{rr}\norm{F} + \norm{F}\lambda_1(\PP X_t \PP\T)\\
    &\quad\quad+ \left((1+\norm{F})^2+\norm{\tilde W_t}^2\norm{A^\dagger}^2\right)(\norm{\PP X_t \PA} + \norm{E_t})\\
    &\le -Y_{rr}\alpha + 2\alpha\delta\norm{Y} + (2\alpha^2+2\alpha+1)\frac{\delta}{3\sqrt\alpha}\norm{Y}\\
    &< 0.
  \end{align*}
  The first equality inserts the expression for $\frac{d}{dt}F_t$ from \Cref{dF}. Next, properties $F_t v = \norm{F_t}u$ and $u\T F_t = \norm{F}v\T$ of the top singular pair $u,v$ are used. The first inequality uses $\tilde W_t = \PP\T\PP \tilde W_t$ by \Cref{tiny_PPtW}. The second inequality bounds products by products of norms, using $\norm{u} = \norm{v} = \norm{J} = 1$. The third inequality applies uses that $X_t-\hat Y$ is positive semi-definite and $\lambda_r(\PA \hat Y \PA) = Y_{rr}$ by \Cref{tiny_PYP}. The fourth inequality applies the assumption $\norm{F_t} = \alpha$, $\norm{\tilde W_t}\norm{A^\dagger} \le \alpha$ by \Cref{tWAd_bound}, $\lambda_1(\PP X_t\PP\T) \le 2\delta\norm{Y}$ by the induction hypothesis \eqref{ind:t} and $\norm{\PP X_t \PA} + \norm{E_t} \le \frac{\delta}{3\sqrt\alpha}\norm{Y}$ by \Cref{EJW_bound}. Finally, use the assumptions $\delta \le \frac{Y_{rr}}{64\sqrt\alpha\norm{Y}}$ by \eqref{delta_bound} and $\alpha \ge 1$ by \eqref{alpha_gamma_bound}.
\end{proof}

\begin{corollary}\label{cor:F}
  There exists $\hat t > t$ such that $\norm{F_\tau} \le \alpha$ for all $\tau \in [t,\hat t]$.
\end{corollary}
\begin{proof}
  Apply \Cref{tech_norm} with $X \coloneqq F$ and $c = \alpha$.
\end{proof}

\begin{proposition}\label{prop:tW}
  Let $u,v$ be a top singular pair of $\tilde W_t$ and assume $\tilde W_t \ne 0$. Then \[\frac{d}{dt}(u\T \tilde W_t v) < 3\sqrt\alpha\delta\norm{Y}\norm{\tilde W_t}.\]
\end{proposition}
\begin{proof}
  Let $X_t \coloneqq \hat Y + \frac{1}{2}JW_tW_t\T J$. Then
  \begin{align*}
    &\frac{d}{dt}(u\T \tilde W_t v)\\
    &= u\T \PP\T(\PP-F_t\PA) \left(X_t + E_t\right) \tilde W_tv - u\T \tilde W_t\tilde W_t\T\left[(X_t + E_t)\PA\T(A_t^\dagger)\T-\frac{1}{2}W_t+\tilde W_t\right]v\\
    &= \norm{\tilde W_t} u\T \PP\T(\PP-F_t\PA) \left(X_t + E_t\right)\PP\T\PP u\\
    &\quad\quad- \norm{\tilde W_t}^2 u\T \PP\T\PP(X_t + E_t)\PA\T(A_t^\dagger)\T v - \frac{1}{2}\norm{\tilde W_t}^3\\
    &\le \left[\lambda_1(\PP X_t \PP\T) + \norm{E_t} + \left(\norm{F_t}+\norm{\tilde W_t}\norm{A_t^\dagger}\right)(\norm{\PP X_t \PA\T} + \norm{E_t})\right]\norm{\tilde W_t}\\
    &\le \left(2\delta\norm{Y}+\frac{\delta^2}{2}\norm{Y} + 2\alpha\frac{\delta}{3\sqrt\alpha}\norm{Y}\right)\norm{\tilde W_t} < 3\sqrt\alpha\delta\norm{Y}\norm{\tilde W_t}.
  \end{align*}
  The first equality expands $\frac{d}{dt}\tilde W_t$ by \Cref{dtW}. Then, \Cref{tiny_PPtW} and properties $\tilde W_tv = \norm{\tilde W_t}u$, $u\T\tilde W_t = \norm{\tilde W_t}v\T$ of the top singular pair $u,v$ are used. Also, note $\tilde W_t Q_t = 0$ so $Q_t v = 0$ and hence $W_t v = \tilde W_t v$. The first inequality bounds products by products of norms. The second inequality applies the established bounds, using \Cref{tWAd_bound}, \Cref{EJW_bound}, \Cref{E_bound}, $\lambda_1(\PP X_t\PP\T) \le 2\delta\norm{Y}$ and $\norm{F_t} \le \alpha$ by the induction hypothesis \eqref{ind:t}. Finally, use $\delta \le \frac{1}{64}$ and $\alpha \ge 1$.
\end{proof}

\begin{corollary}\label{cor:tW}
  There exists $\hat t > t$ such that $\norm{\tilde W_\tau} \le \norm{W_0} \exp\left(3\sqrt\alpha \delta \norm{Y}\tau\right)$ for all $\tau \in [t,\hat t]$.
\end{corollary}
\begin{proof}
  Apply \Cref{tech_norm} with $X_\tau \coloneqq \exp(-3\sqrt\alpha\delta\norm{Y}\tau)\tilde W_\tau$ and $c = \norm{W_0}$. We verify the requirements on $X_t$. First, $\norm{X_t} \le c$ by the induction hypothesis \eqref{ind:t}. Second, if $\norm{X_t} = c > 0$, then for a top singular pair $u,v$ of $X_t$ we have \[u\T \dot X_t v = \left(-3\sqrt\alpha\delta\norm{Y} u\T \tilde W_t v + \frac{d}{dt} (u\T \tilde W_t v)\right) \exp(-3\sqrt\alpha\delta\norm{Y}t) < 0.\] The last inequality used $u\T \tilde W_t v = \norm{\tilde W_t}$ and \Cref{prop:tW}.
\end{proof}

\begin{proposition}\label{prop:A}
  Assume $0 < \sigma_r(A_t) \le \sqrt{Y_{rr}}$. Let $u,v$ be a bottom singular pair of $A_t$. Then $\frac{d}{dt}(u\T A_t v) > \frac{2Y_{rr}}{5}\sigma_r(A_t)$.
\end{proposition}
\begin{proof}
  \begin{align}
    \frac{d}{dt}(u\T A_t v) &= u\T \PA \left(\hat Y + E_t - \frac{1}{2}(W_tW_t\T-JW_tW_t\T J)\right)W_t v\nonumber\\
                            &\ge u\T \PA \left(\hat Y\PA\T\PA-\frac{1}{2}W_tW_t\T\right)W_tv - \left(\norm{E_t}+\frac{1}{2}\norm{\PA JW_t}\norm{W_t}\right)\norm{W_tv}\nonumber\\
                            &= \sigma_r(A_t)\left(u\T \PA \hat Y\PA\T u - \frac{1}{2}\norm{W_t v}^2\right) - \left(\norm{E_t}+\frac{1}{2}\norm{\PA JW_t}\norm{W_t}\right)\norm{W_tv}.\label{eq:dA}
  \end{align}
  The first equality uses the definition $A_t = \PA W_t$ and expands $\dot W_t$. The first inequality bounds products by products of norms, and uses $\PA \hat Y = \PA \hat Y \PA\T \PA$ by \Cref{tiny_PPYP}. The last equality uses $\PA W_t = A_t$ and that $u,v$ is a bottom singular pair of $A_t$.

  Furthermore, $\norm{W_t v}^2 = \norm{\PA W_t v}^2 + \norm{\PP W_t v}^2 \le \sigma_r^2(A_t) + \norm{\PP W_t}^2$. We may also bound 
  \begin{align*}
    \norm{W_t v} &= \norm{(\PA\T\PA+\PP\T\PP)W_tv} = \norm{(\PA\T + \PP\T F_t) A_t v}\\
                 &\le (\norm{F_t}+1)\norm{A_tv} = (\norm{F_t}+1)\sigma_r(A_t).
  \end{align*}
The first equality uses $\PA\T\PA+\PP\T\PP = I$. Next, note $Q_tv = v$ since $Q_t$ is the projection onto the row space of $A_t$ and $v$ is a right singular value of $A_t$. The second equality uses the definitions $A_t = \PA W_t$, $F_t = \PP W_t A_t^\dagger$ and $W_t v = W_t Q_t v = W_t A_t^\dagger A_t v$. The inequality uses $\norm{\PA},\norm{\PP} = 1$.

  Using the previous two bounds, we can bound \eqref{eq:dA}
  \begin{align*}
    &\frac{d}{dt}(u\T A_t v)\\
    &\ge \sigma_r(A_t)\left(u\T \PA \hat Y\PA\T u - \frac{1}{2}\norm{W_t v}^2\right) - \left(\norm{E_t}+\frac{1}{2}\norm{\PA JW_t}\norm{W_t}\right)\norm{W_tv}\\
    &\ge \left[Y_{rr} - \frac{1}{2}\left(\sigma_r^2(A_t) + \norm{\PP W_t}^2\right) - \left(\norm{E_t}+\frac{1}{2}\norm{\PA JW_t}\norm{W_t}\right)(\norm{F_t}+1)\right]\sigma_r(A_t)\\
    &\ge \left[Y_{rr} - \frac{1}{2}\left(Y_{rr} + 9\delta^2\norm{Y}\right) - \frac{\delta}{3\sqrt\alpha}\norm{Y}(\alpha+1)\right]\sigma_r(A_t) > \frac{2Y_{rr}}{5}\sigma_r(A_t).
  \end{align*}
  For the second inequality, use $\lambda_r(\PA \hat Y\PA\T) = Y_{rr}$ by \Cref{tiny_PYP}, and the bounds derived above for $\norm{W_tv}^2$ and $\norm{W_tv}$. The third inequality uses the assumption $\sigma_r^2(A_t) \le Y_{rr}$, $\norm{\PP W_t} \le 3\delta\sqrt{\norm{Y}}$ by \Cref{PPW_bound}, $\norm{F_t} \le \alpha$ by the induction hypothesis \eqref{ind:t} and $\norm{E_t}+\frac{1}{2}\norm{\PA JW_t}\norm{W_t} \le \frac{\delta}{3\sqrt\alpha}\norm{Y}$ by \Cref{EJW_bound}. Finally, use the assumptions $\delta \le \frac{Y_{rr}}{64\sqrt\alpha\norm{Y}}$ by \eqref{delta_bound} and $\alpha \ge 1$ by \eqref{alpha_gamma_bound}.
\end{proof}

\begin{corollary}\label{cor:A}
  There exists $\hat t > t$ such that for $\tau \in [t,\hat t]$, we have \[\sigma_r(A_\tau) \ge \min\left(\sqrt{Y_{rr}},\ \frac{\norm{W_0}}{\alpha} \exp\left(\frac{2Y_{rr}}{5}\tau\right)\right).\]
\end{corollary}
\begin{proof}
  If $\sigma_r(A_t) > \sqrt{Y_{rr}}$, such a $\hat t$ exists by continuity of $\sigma_r(A_t)$. Thus, we assume $\sigma_r(A_t) \le \sqrt{Y_{rr}}$, which allows us to apply \Cref{prop:A}. Apply \Cref{tech_eig} with $X_\tau \coloneqq -\exp\left(-\frac{4Y_{rr}}{5}\tau\right)A_\tau A_\tau\T$ and $c = -\exp\left(-\frac{4Y_{rr}}{5}t\right)\sigma_r^2(A_t)$. By definition of $c$, $\lambda_1(X_t) = c$. Furthermore, whenever $u \in \R^r$ satisfies $\norm{u} = 1$ and $u\T X_t u = c$, it is an eigenvector of $X_t$. $u$ is also part of a bottom singular pair $u,v$ of $A_t$, and we have 
  \begin{align*}
    u\T \dot X_t u &= e^{-\frac{4Y_{rr}}{5}t}\left(\frac{4Y_{rr}}{5}u\T A_tA_t\T u-2u\T \dot A_t A_t\T u\right)\\
                   &= e^{-\frac{4Y_{rr}}{5}t}\left(\frac{4Y_{rr}}{5}\sigma_r^2(A_t)-2\sigma_r(A_t)u\T \dot A_t v\right) < 0.
  \end{align*}
  Hence, \Cref{tech_eig} gives the existance of $\hat t > t$ such that $\lambda_1(X_\tau) \le c$ for $\tau \in [t,\hat t]$. In terms of $A_t$, that is $\sigma_r^2(A_\tau) \ge \exp\left(\frac{4Y_{rr}}{5}(\tau-t)\right)\sigma_r^2(A_t)$. When combined with the induction hypothesis \eqref{ind:t} for $\sigma_r(A_t)$, it implies the desired inequality.
\end{proof}

\subsection{Proofs for local convergence}\label{s:local}
After time $T_1 = \frac{5}{4Y_{rr}}\log\left(\alpha^2\frac{Y_{rr}}{\norm{W_0}^2}\right)$, we have $\sigma_r(A_t) \ge \sqrt{Y_{rr}}, t \ge T_1$, which leads to rapid decrease in the residual $R = \hat Y-\frac{1}{2}\left(WW\T -JWWJ\T\right)$. This is formalized by the following.
\begin{definition}\label{ind2:def}
  Let $S^{local}$ be the set of times $t \in [T_1,T_2]$ such that
  \begin{enumerate}
    \item $\norm{R_t} \le M^R_t$,\label{ind2:R}
    \item $\norm{\PN W_tQ_t} \le \frac{2}{5}M^R_t / \sqrt{\norm{Y}}$,\label{ind2:NWQ}
  \end{enumerate}
  where
  \begin{align}
    M^R_t &= \max\left(3\norm{Y}\exp\left(-\frac{2Y_{rr}}{5}(t-T_1)\right), M^R_\infty\right)\label{MR_def},\\
    M^R_\infty &= 64\left(\beta\gamma\frac{\norm{Y}}{Y_{rr}}+\sqrt{\frac{\norm{Y}}{Y_{rr}}}\right)\norm{W_0}^2\exp(6\sqrt\alpha \delta \norm{Y}T_2) + 10^3\mu\frac{\norm{Y}^2}{Y_{rr}}\label{Minf_def}.
  \end{align}
\end{definition}
Note that the bounds in \Cref{ind:def} also hold until time $T_2$. The goal of this sub-section is to prove that the bounds in \Cref{ind2:def} hold at all times $T_1 \le t \le T_2$.

\begin{theorem}\label{ind2:thm}
  $S^{local} = [T_1,T_2]$. That is, the conditions of \Cref{ind2:def} hold for all times $t \in [T_1,T_2]$.
\end{theorem}

As the following lemma shows, \Cref{ind2:thm} at time $T_2$ yields $\norm{R_{T_2}} \le M^R_\infty$. This implies \eqref{goal}, which was our goal for proving \Cref{master_theorem}. Hence, we just need to prove \Cref{ind2:thm}.

\begin{lemma}\label{tiny_RT2}
  \Cref{ind2:thm} implies $\norm{R_{T_2}} \le M^R_\infty$.
\end{lemma}
\begin{proof}
  We have $T_2 = \frac{5}{Y_{rr}}\log\left(\frac{Y_{rr}}{\norm{W_0}^2}\right)$ by \eqref{T2_bound}. Then \Cref{ind2:thm} at time $T_2$ yields 
  \[\norm{R_{T_2}} \le M^R_{T_2} = \max\left(3\norm{Y}\exp\left(-\frac{2Y_{rr}}{5}(T_2-T_1)\right), M^R_\infty\right).\] Note that by the assumptions (\ref{eps_bound},~\ref{delta_bound}), we have $\norm{W_0} \le \frac{\delta}{3\sqrt\alpha}\sqrt{\norm{Y}}$ and $\delta \le \frac{Y_{rr}}{\norm{Y}\sqrt\alpha}$. This means $\frac{Y_{rr}}{\norm{W_0}^2} \ge \frac{\alpha^2 \norm{Y}}{Y_{rr}}$. By \eqref{T1_bound}, $T_1 = \frac{5}{4Y_{rr}}\log\left(\alpha^2\frac{Y_{rr}}{\norm{W_0}^2}\right)$, this means 
  \[3\norm{Y}\exp\left(-\frac{2Y_{rr}}{5}(T_2-T_1)\right) = 3\norm{Y}\alpha\left(\frac{\norm{W_0}^2}{Y_{rr}}\right)^\frac{3}{2} \le 3\sqrt{\frac{\norm{Y}}{Y_{rr}}}\norm{W_0}^2 \le M^R_\infty.\] Hence, $\norm{R_{T_2}} \le M^R_{T_2} = M^R_\infty$.
\end{proof}

The proof of \Cref{ind2:thm} is similar in structure to the proof of \Cref{ind:thm}. Again we use real induction (\Cref{real_ind}). The base case (\Cref{local:base}) and continuity property (\Cref{local:cont}) are straightforward. We prove the induction steps for $\norm{R_t}$ and $\norm{\PN W_tQ_t}$ separately.

\begin{proposition}[Base case]\label{local:base}
  $[0,T_1] \in S^{local}$.
\end{proposition}
\begin{proof}
  At times $t \in [0,T_1]$, we have $M^R_t \ge 3\norm{Y}$. By \Cref{ind:thm}, we also have $\norm{W_t} \le \frac{3}{2}\sqrt{\norm{Y}}$. Then \Cref{tiny_R_bound} gives $\norm{R_t} \le \frac{17}{8}\norm{Y} \le M^R_t$. Additionally, $\norm{\PN W_t Q_t} \le \norm{\PN W_t} \le \delta\sqrt{8\norm{Y}} < \frac{2M^R_t}{5\sqrt{\norm{Y}}}$.
\end{proof}

\begin{proposition}[Continuity]\label{local:cont}
  If $t > 0$ and $[0,t) \in S^{local}$, then $t \in S^{local}$.
\end{proposition}
\begin{proof}
  This follows from all expressions in \Cref{ind2:def} being continuous with respect to $t$.
\end{proof}

\begin{proposition}[Induction step]\label{local:step}
  Let $t \in [T_1,T_2)$ and assume 
  \begin{align}\label{ind2:t}
    [T_1,t] \subset S^{local},
  \end{align}
  then there exists $\hat t > t$ such that $[t,\hat t] \subset S^{local}$.
\end{proposition}
\begin{proof}
  We pick $\hat t$ as the minimum of the ones provided in \Cref{cor:R,cor:NWQ}.
\end{proof}

Throughout the rest of this section, we prove \Cref{cor:R,cor:NWQ} in the setting of \Cref{local:step}. We fix a $t \in [T_1,T_2)$ such that \eqref{ind2:t} holds. That is, \Cref{ind2:R,ind2:NWQ} in \Cref{ind2:def} hold at time $t$. Before proving the induction steps for $R_t$ and $\PN W_t Q_t$, we show three consequences of the induction hypothesis.

\begin{lemma}\label{E_bound2}
  $\norm{E_t} \le \frac{Y_{rr}}{56\norm{Y}}M^R_t.$
\end{lemma}
\begin{proof}
  Recall the assumed bound on $\norm{E_t}$ from \eqref{new_general_E_bound}:
  \begin{align}\label{eq:E_bound2}
    \norm{E_t} \le \beta\left(\norm{R_\tau}+\gamma\sigma_{r+1}^2(W_\tau)\right) + \mu\norm{Y}.
  \end{align}
  By \Cref{tiny_tWrW} and \Cref{ind:thm}, we have $\sigma_{r+1}^2(W_\tau) \le \norm{\tilde W_\tau}^2 \le \norm{W_0}^2 \exp\left(6\sqrt\alpha \delta \norm{Y}T_2\right)$. By the induction hypothesis \eqref{ind2:t}, we have $M^R_t \ge 64\left(\beta\gamma\frac{\norm{Y}}{Y_{rr}}+\sqrt{\frac{\norm{Y}}{Y_{rr}}}\right)\norm{W_0}^2\exp(6\sqrt\alpha \delta \norm{Y}T_2)$. Together, these imply $\beta\gamma\sigma_{r+1}^2(W_\tau) \le \frac{M^R_t Y_{rr}}{64\norm{Y}}$.

  Recall the assumptions $\delta \le \frac{Y_{rr}}{64\norm{Y}}$ by \eqref{delta_bound}, $\beta \le \frac{\delta^2}{13}$ by \eqref{beta_bound} and $\eta \le \frac{1}{Y_{rr}}$ by \eqref{eta_bound}. Hence, we have $\beta \le \frac{Y_{rr}}{10^4\norm{Y}}$. Furthermore, the induction hypothesis also bounds $R_\tau$ as follows
  \begin{align*}
    \beta\norm{R_\tau} \le \beta\exp\left(\frac{2Y_{rr}}{5}\eta\right)M^R_t \le \frac{Y_{rr}}{10^3\norm{Y}}M^R_t.
  \end{align*}
  Finally, we bound $\mu\norm{Y}$. The induction hypothesis \eqref{ind2:t} yields $M^R_t \ge M^R_\infty \ge 10^3\mu\norm{Y}^2/Y_{rr}$. This means $\mu\norm{Y} \le \frac{Y_{rr}}{10^3\norm{Y}}M^R_t$.

  Using the three bounds developed above on the terms of \eqref{eq:E_bound2}, we obtain the desired inequality.
\end{proof}

\begin{lemma}\label{tW_bound2}
  $\norm{\tilde W_t}^2 \le \sqrt{\frac{Y_{rr}}{\norm{Y}}}\frac{M^R_t}{64}$.
\end{lemma}
\begin{proof}
  By \Cref{ind:thm}, we have $\norm{\tilde W_t} \le \norm{W_0}\exp(3\sqrt\alpha\delta\norm{Y}T_2)$. Squaring, and comparing to the definition of $M^R_\infty$ \eqref{Minf_def}, that is $\norm{\tilde W_t}^2 \le \frac{1}{64}\left(\beta\gamma\frac{\norm{Y}}{Y_{rr}}+\sqrt{\frac{\norm{Y}}{Y_{rr}}}\right)^{-1} M^R_\infty$. Noting $M^R_t \ge M^R_\infty$ by \eqref{MR_def}, this implies the desired inequality.
\end{proof}

\begin{lemma}\label{A2}
  $\sigma_r(A_t) \ge \sqrt{Y_{rr}}$.
\end{lemma}
\begin{proof}
  Immediate consequence of $\sigma_r(A_t) \ge \min\left(\sqrt{Y_{rr}},\ \frac{\norm{W_0}}{\alpha} \exp\left(\frac{2 Y_{rr}}{5}t\right)\right)$ by \Cref{ind:thm} and $t \ge T_1 = \frac{5}{4Y_{rr}}\log\left(\alpha^2\frac{Y_{rr}}{\norm{W_0}^2}\right)$.
\end{proof}

We are now ready to prove the induction steps for $\norm{R_t}$ and $\norm{\PN W_t Q_t}$. Note $\norm{R_t} = \lambda_1(R_t)$ by \Cref{tiny_eigR}.

\begin{proposition}\label{prop:R}
  Assume $\lambda_1(R_t) = M^R_t$. Let $v$ be an eigenvector corresponding to the largest eigenvalue of $R_t$. Then $\frac{d}{dt}v\T R_t v < -\frac{2Y_{rr}}{5}\norm{R_t}$.
\end{proposition}
\begin{proof}
  \begin{align}
    \frac{d}{dt}v\T R_t v &= -\frac{1}{2}v\T \left(\dot W_t W_t\T + W_t \dot W_t\T - J\dot W_tW_t\T J - JW_t\dot W_t\T J\right)v\nonumber\\
                          &= -v\T \left(\dot W_t W_t\T - J\dot W_tW_t\T J\right)v\nonumber\\
                          &= -v\T \left((R_t+E_t)W_t W_t\T - J(R_t+E_t) W_tW_t\T J\right)v\nonumber\\
                          &\le -v\T R_t \left(W_t W_t\T + JW_tW_t\T J\right)v + 2\norm{E_t}\norm{W_t}^2\nonumber\\
                          &= -\norm{R_t} \left(\norm{v\T W_t}^2+\norm{v\T JW_t}^2\right) + 2\norm{E_t}\norm{W_t}^2.\label{dR1}
  \end{align}
  The first equality expands the definition of $R_t$ and applies the product rule. The second equality groups terms using symmetry. The third equality expands $\dot W_t$. The inequality bounds products by products of norms and uses $\norm{u} = \norm{v} = \norm{J} = 1$. The final equality uses $v\T R_t = \lambda_1(R_t)v\T = \norm{R_t}v\T$ by $v$ being an eigenvector and \Cref{tiny_eigR}.

  Next, we split $v = v\T\PA\T\PA + v\T\PP\T\PP$. Then
  \begin{align*}
    \norm{v\T W_t}^2 &= \norm{v\T\PA\T\PA W_t + v\T\PP\T\PP W_t}^2\\
                     &\ge \norm{v\T\PA\T A_t}^2 - 2\norm{v\T\PA\T}\norm{v\T\PP\T}\norm{\PP W_t}\norm{W_t}\\
                     &\ge \sigma_r^2(A_t)\norm{v\T\PA\T}^2 - \norm{\PP W_t}\norm{W_t}.
  \end{align*}
  The first inequality follows from $\norm{x+y}^2 \ge \norm{x}^2-2\norm{x}\norm{y}$ for vectors $x,y$. The second equality uses $2\norm{v\T\PA\T}\norm{v\T\PP\T} \le \norm{v\T\PA\T}^2+\norm{v\T\PP\T}^2 = \norm{v}^2 = 1$, and that $\sigma_r(A_t)$ is the least singular value of $A_t$.

  Substituting $v\T \to v\T J$, we also get $\norm{v\T J W_t}^2 \ge \sigma_r^2(A_t)\norm{v\T J\PA\T}^2 - \norm{\PP W_t}\norm{W_t}$. In combination, we may bound \eqref{dR1} as follows:
  \begin{align*}
    \frac{d}{dt}v\T R_t v &\le -\norm{R_t} \left(\sigma_r^2(A_t)\left(\norm{v\T\PA\T}^2+\norm{v\T J\PA\T}^2\right) - 2\norm{\PP W_t}\norm{W_t}\right)\\
    &\quad\quad+ 2\norm{E_t}\norm{W_t}^2.
  \end{align*}

  Next, we bound
  \begin{align}
    \norm{R_t}^2 &= \norm{v\T R_t}^2 = \norm{v\T R_t\PN\T}^2+\norm{v\T R_t\PA\T}^2+\norm{v\T R_t J\PA\T}^2\nonumber\\
                 &\le \norm{\PN R_t}^2+\norm{R_t}^2\left(\norm{v\T\PA\T}^2+\norm{v\T J\PA\T}^2\right).\label{NR_bound}
  \end{align}
  The first equality is a property of the eigenvector $v$. The second equality uses \Cref{tiny_PNPJ}. The inequality again uses $\norm{R_t}v\T = v\T R_t$ and $\norm{v} = 1$.

  Furthermore, 
  \begin{align*}
    \norm{\PN R_t} &= \frac{1}{2}\norm{\PN W_tW_t\T-\PN JW_tW_t\T J} \le \norm{\PN W_t W_t\T}\\
                   &= \norm{\PN W_t Q_t W_t\T + \PN \tilde W_t \tilde W_t\T} \le \norm{\PN W_t Q_t}\norm{W_t} + \norm{\tilde W_t}^2\\
                   &\le \left(\frac{3}{5}+\frac{1}{64}\right)M^R_t = \frac{197}{320}\norm{R_t}.
  \end{align*}
  The first equality expands the definition of $R_t$ and cancels $\PN \hat Y = 0$. The first inequality uses $\norm{J} = 1$ and the triangle inequality. The next equality uses \Cref{tiny_WQW}. Then, the triangle inequality is used. The penultimate inequality uses $W_t \le \frac{3}{2}\sqrt{\norm{Y}}$ by \Cref{ind:thm}, $\norm{\PN W_t Q_t} \le \frac{2}{5}M^R_t / \sqrt{\norm{Y}}$ by the induction hypothesis \eqref{ind2:t} and $\norm{\tilde W_t}^2 \le \frac{M^R_t}{64}$ by \Cref{tW_bound2}. The final equality uses the assuption on $\norm{R_t}$ and $\lambda_1(R_t) = \norm{R_t}$ by \Cref{tiny_eigR}.

  Combining with \eqref{NR_bound}, we conclude that $\norm{\PP v}^2+\norm{\PP Jv}^2 \ge 1-\frac{\norm{\PN R_t}^2}{\norm{R_t}^2} \ge 1-\frac{197^2}{320^2}$. Hence, we can bound
  \begin{align*}
    \frac{d}{dt}v\T R_t v &\le -\norm{R_t} \left(\left(1-\frac{\norm{\PN R_t}^2}{\norm{R_t}^2}\right)\sigma_r^2(A_t) - 2\norm{\PP W_t}\norm{W_t}\right) + 2\norm{E_t}\norm{W_t}^2\\
                          &< -\frac{2Y_{rr}}{5}\norm{R_t},
  \end{align*}
  where we bound $\sigma_r(A_t) \ge \sqrt{Y_{rr}}$ by \Cref{A2}, $\norm{E_t} \le \frac{Y_{rr}}{56\norm{Y}}M^R_t = \frac{Y_{rr}}{56\norm{Y}}\norm{R_t}$ by \Cref{E_bound2}, $\norm{\PP W_t} \le 3\delta\sqrt{\norm{Y}}$ by \Cref{PPW_bound} and $W_t \le \frac{3}{2}\sqrt{\norm{Y}}$ by \Cref{ind:thm}.
\end{proof}

\begin{corollary}\label{cor:R}
  There exists $\hat t > t$ such that $\norm{R_\tau} \le M^R_\tau$ for all $\tau \in [t,\hat t]$.
\end{corollary}
\begin{proof}
  By \Cref{tiny_eigR}, it is enough to bound $\lambda_1(R_\tau)$. If $\norm{R_t} < M^R_t$, such a $\hat t$ exists by continuity. Hence, we can assume $\norm{R_t} = M^R_t$ and apply \Cref{tech_eig} with $X_\tau \coloneqq e^{-\frac{2Y_{rr}}{5}\tau}R_\tau$ and $c = e^{-\frac{2Y_{rr}}{5}t}\lambda_1(R_t)$. Using \Cref{prop:R} to satisfy the assumptions, we can hence find $\hat t > t$ such that $\lambda_1(R_\tau) \le e^{-\frac{2Y_{rr}}{5}(\tau-t)}\lambda_1(R_t)$ for $\tau \in [t,\hat t]$.
\end{proof}

\begin{proposition}\label{prop:NWQ}
  Assume $\norm{\PN W_t Q_t} = \frac{2}{5\sqrt{\norm{Y}}}M^R_t$ and let $u,v$ be a top singular pair of $\PN W_t Q_t$. Then $\frac{d}{dt}u\T \PN W_t Q_t v < -\frac{2Y_{rr}}{5}\norm{\PN W_t Q_t}$.
\end{proposition}
\begin{proof}
  By the definition $\tilde W_t = W_t(I-Q_t)$, we have $\PN W_t Q_t = \PN W_t - \PN \tilde W_t$. We will bound the two terms separately.
  \begin{align*}
   &\frac{d}{dt}\left(u\T \PN W_t v\right)\\
    &= u\T \PN \left(\hat Y + E_t - \frac{1}{2}(W_tW_t-JW_tW_t\T J)\right)W_tv\\
                                &= u\T \PN \left(E_t - \frac{1}{2}(W_tQ_tW_t\T+\tilde W_t\tilde W_t\T-JW_tQ_tW_t\T J-J\tilde W_t\tilde W_t\T J)\right)W_tv\\
                                &\le -\frac{1}{2}u\T \PN W_t Q_t W_t\T W_tv+\frac{1}{2}\norm{\PN JW_tQ_t}\norm{\Bt t} + \left(\norm{E_t}+\frac{1}{2}\norm{\tilde W_t}^2\right)\norm{W_t v}\\
                                &= -\frac{1}{2}\norm{\PN W_t Q_t}\left(\norm{W_tv}^2-\frac{1}{2}\norm{\Bt t}\right) + \left(\norm{E_t}+\frac{1}{2}\norm{\tilde W_t}^2\right)\norm{W_t v}.
  \end{align*}
  The first equality expands $\dot W_t$. The second equality uses \Cref{tiny_WQW}. The inequality uses \Cref{tiny_XmJXJ} and bounds products by products of norms, using $\norm{u} = \norm{v} = \norm{J} = 1$. The final equality uses $u\T\PN W_t Q_t = \norm{\PN W_t Q_T}v\T$ and $\norm{\PN J W_t Q_t} = \norm{\PN W_t Q_t}$ by \Cref{tiny_NJX}.

  Next, we bound
  \begin{align*}
    \frac{d}{dt}\left(-u\T \PN \tilde W_t v\right) &= -u\T \PN \tilde W_t\tilde W_t\T \PP\T\PP \left[(X_t + E_t)\PA\T(A_t^\dagger)\T-\frac{1}{2}W_t\right]v\\
                                       &\le \norm{\tilde W_t}^2\left(\left(\norm{\PP X_t \PA\T} + \norm{E_t}\right)\norm{A^\dagger}+\frac{1}{2}\norm{W_tv}\right).
  \end{align*}
  The first equality expands $\frac{d}{dt}\tilde W_t$ from \Cref{dtW}, removes terms containing $\tilde W_t v = 0$ (proved below), and applies $\PP\T\PP\tilde W_t = \tilde W_t$ by \Cref{tiny_PPtW}. The inequality bounds products by products of norms.

  To see that $\tilde W_t v = 0$, note that $\norm{NW_tQ_t} = \frac{2}{5\sqrt{Y}}M^R_t > 0$ and \[\norm{NW_tQ_t}v\T \tilde W_t\T = u\T N W_t Q_t (I-Q_t) W_t\T = 0.\] Here we used that $u,v$ is a top singular pair of $NW_tQ_t$, $\tilde W_t = W_t(I-Q_t)$ and $Q_t^2 = Q_t$ for the projection $Q_t$.

  Next, we add the bounds on $\frac{d}{dt}\left(u\T \PN W_t v\right)$ and $\frac{d}{dt}\left(-u\T \PN \tilde W_t v\right)$, to get
  \begin{align*}
    &\frac{d}{dt}(u\T \PN W_t Q_t v)\\
    &\le -\frac{1}{2}\norm{\PN W_t Q_t}\left(\norm{W_tv}^2-\frac{1}{2}\norm{\Bt t}\right) + \left(\norm{E_t}+\norm{\tilde W_t}^2\right)\norm{W_t v}\\
    &\quad\quad + \norm{\tilde W_t}^2\norm{A^\dagger}\left(\norm{\PP X_t \PA\T} + \norm{E_t}\right)\\
    &\le -\norm{W_tv}^2\Biggl(\norm{\PN W_t Q_t}\left(\frac{1}{2}-\frac{\norm{\Bt t}}{4Y_{rr}}\right) - \frac{\norm{E_t}+\norm{\tilde W_t}^2}{\sqrt{Y_{rr}}}\\
    &\quad\quad- \norm{\tilde W_t}^2\frac{\norm{\PP X_t \PA\T} + \norm{E_t}}{Y_{rr}^{3/2}}\Biggr)\\
    &\le -\norm{W_tv}^2\Biggl(\norm{\PN W_t Q_t}\left(\frac{1}{2}-\frac{\norm{\Bt t}}{4Y_{rr}}\right) - \frac{\frac{1}{56}+\frac{1}{64}}{\sqrt{\norm{Y}}}M^R_t\\
    &\quad\quad- \frac{\norm{\PP X_t \PA\T} + \norm{E_t}}{64Y_{rr}\sqrt{\norm{Y}}}M^R_t\Biggr)\\
                                    &< -\frac{2Y_{rr}}{5}\norm{\PN W_t Q_t}.
  \end{align*}
  The second inequality uses $\sigma_r(A_t) \ge \sqrt{Y_{rr}}$ by \Cref{A2}, $\norm{A_t^\dagger} = \frac{1}{\sigma_r(A_t)}$ and \[\norm{W_tv} \ge \norm{\PA W_t v} = \norm{A_t Q_t v} \ge \sigma_r(A_t).\] The third inequality uses $\norm{E_t} \le \frac{Y_{rr}}{56\norm{Y}}M^R_t$ by \Cref{E_bound2} and $\norm{\tilde W_t}^2 \le \sqrt{\frac{Y_{rr}}{\norm{Y}}}\frac{M^R_t}{64}$ by \Cref{tW_bound2}. The final inequality uses $\norm{W_tv}^2 \ge Y_{rr}$ again, $\norm{\Bt t} \le 6\delta^2\norm{Y}$ by \Cref{ind:thm}, the assumption $\norm{\PN W_t Q_t} = \frac{2}{5\sqrt{\norm{Y}}}M^R_t$, $\norm{\PP X_t \PA\T} + \norm{E_t} \le \frac{\delta}{3\sqrt\alpha}\norm{Y}$ by \Cref{EJW_bound} and $\delta \le \frac{Y_{rr}}{64\norm{Y}}$.
\end{proof}

\begin{corollary}\label{cor:NWQ}
  There exists $\hat t > t$ such that $\norm{\PN W_\tau Q_\tau} \le \frac{2}{5\sqrt{\norm{Y}}}M^R_\tau$ for all $\tau \in [t,\hat t]$.
\end{corollary}
\begin{proof}
  If $\norm{\PN W_t Q_t} < \frac{2}{5\sqrt{\norm{Y}}}M^R_t$, such a $\hat t$ exists by continuity. Hence, we can assume $\norm{R_t} \ge \frac{2}{5\sqrt{\norm{Y}}}M^R_t$ and apply \Cref{tech_norm} with $X_\tau \coloneqq e^{-\frac{2Y_{rr}}{5}\tau}\PN W_\tau Q_\tau$ and $c = \norm{X_t}$. Using \Cref{prop:NWQ} to satisfy the requirements of \Cref{tech_norm}, we can find $\hat t > t$ such that $\norm{\PN W_\tau Q_\tau} \le e^{-\frac{2Y_{rr}}{5}(\tau-t)}\norm{\PN W_t Q_t}$ for $\tau \in [t,\hat t]$.
\end{proof}

\subsection{Some consequences of the setup}\label{s:conseq}
We summarize some useful direct consequences of the definitions in \Cref{s:setup}, for easy reference in the proofs.

\begin{lemma}\label{tiny_dilation}
  \begin{alignat*}{3}
    \norm{\hat Y} &= \norm{Y}, \quad && J\hat Y J = -\hat Y, \quad && J\hat Y+\hat Y J = 0,\\
    \norm{R} &= \norm{Y-UV\T}, \quad && JRJ = -R, \quad && JR+RJ = 0.
  \end{alignat*}
\end{lemma}
\begin{proof}
  Recall $\hat Y \coloneqq \begin{pmatrix} 0 & Y \\ Y\T & 0\end{pmatrix}$ and $R \coloneqq \begin{pmatrix} 0 & Y-UV\T \\ Y\T-VU\T & 0\end{pmatrix}$. That is, the expressions in the norms in the statement are such that the left-hand sides are defined to be the \textit{self-adjoint dilation} of the right-hand sides. The self-adjoint dilation preserves norms. Recall $J = \begin{pmatrix} I_{n_1} & 0 \\ 0 & -I_{n_2} \end{pmatrix}$. The remaining identites can be verified by calculating the matrix products.
\end{proof}

\begin{lemma}\label{tiny_XmJXJ}
  For $n_1+n_2\times n_1+n_2$ matrices $X$, we have $\norm{X-JXJ} \le 2\norm{X}$. If $X$ is symmetric positive semi-definite, then $\norm{X-JXJ} \le \norm{X}$.
\end{lemma}
\begin{proof}
  The first inequality follows from $\norm{J} = 1$ and the triangle inequality. The second inequality follows from \[\norm{X-JXJ} = \max\left(\lambda_1(X-JXJ), \lambda_1(JXJ-X)\right) \le \max\left(\lambda_1(X), \lambda_1(JXJ)\right) \le \norm{X}.\]
\end{proof}

\begin{lemma}\label{tiny_R_bound}
  $\norm{R} \le \norm{Y}+\frac{1}{2}\norm{W}^2$. Specifically, if $\norm{W} \le \frac{3}{2}\sqrt{\norm{Y}}$, then $\norm{R} \le \frac{17}{8}\norm{Y}$.
\end{lemma}
\begin{proof}
  \begin{align*}
    \norm{R} = \norm{\hat Y-\frac{1}{2}\left(WW\T-JWW\T J\right)} \le \norm{\hat Y}+\frac{1}{2}\norm{WW\T-JWW\T J} \le \norm{Y}+\frac{1}{2}\norm{W}^2.
  \end{align*}
  The equality expands the definition of $R$, then the triangle inequality is used. The final inequality uses \Cref{tiny_dilation,tiny_XmJXJ}.
\end{proof}

\begin{lemma}\label{tiny_eigR}
  $\lambda_1(R) = \norm{R}$.
\end{lemma}
\begin{proof}
  $R$ is symmetric and $\lambda_1(-R) = \lambda_1(JRJ) = \lambda_1(R)$ by \Cref{tiny_dilation}.
\end{proof}

\begin{lemma}\label{tiny_PPtW}
  $\PP\T\PP\tilde W = \tilde W$ and $\PA\tilde W = 0$.
\end{lemma}
\begin{proof}
  Note the definitions $A = \PA W$, $Q = A^\dagger A$ and $\tilde W = W(I-Q)$ imply $\PA\T \PA \tilde W = \PA\T \PA W (I-Q) = 0$. Then $\PA\tilde W = 0$ since $\PA\PA\T = I$. Furthermore, $\PP\T\PP\tilde W = (I-\PA\T\PA) \tilde W = \tilde W$.
\end{proof}

\begin{lemma}\label{tiny_WQW}
  $W W \T = W Q W \T + \tilde W \tilde W\T$.
\end{lemma}
\begin{proof}
  $\tilde W W\T = W (I-Q) W\T = W (I-Q)^2 W\T = \tilde W \tilde W\T$.
\end{proof}

\begin{lemma}\label{tiny_WAd}
  $W A^\dagger = \PA\T+\PP\T F$.
\end{lemma}
\begin{proof}
  Note $\PA\T\PA+\PP\T\PP = I$, so
  \begin{align*}
    W A^\dagger = (\PA\T\PA+\PP\T\PP)W A^\dagger = \PA\T AA^\dagger+\PP\T\PP W A^\dagger = \PA\T +\PP\T F.
  \end{align*}
  The second equality uses the definition $A = \PA W$. The final equality uses $AA^\dagger = I$ and the definition $F = \PP W A^\dagger$.
\end{proof}

\begin{lemma}\label{tiny_PYP}
  $\lambda_r(\PA \hat Y \PA\T) = Y_{rr}$.
\end{lemma}
\begin{proof}
  $\PA \hat Y \PA\T$ is diagonal with diagonal equal to the top $r$ eigenvalues of $\hat Y$, which are equal to the top $r$ singular values of $Y$.
\end{proof}

\begin{lemma}\label{tiny_NJX}
  $\norm{\PN J X} = \norm{\PN X}$ for $X \in \R^{(n_1+n_2)\times d}$.
\end{lemma}
\begin{proof}
  By definition of $\PN$ and $J$, we have \[\norm{\PN J X} = \norm{\begin{matrix}X_{r+1:m,:} \\ -X_{m+r+1:n_1+n_2,:}\end{matrix}} = \norm{\begin{matrix}X_{r+1:m,:} \\ X_{m+r+1:n_1+n_2,:}\end{matrix}} = \norm{\PN X}.\] Here $X_{a:b,:}$ extracts rows $a$ through $b$ of $X$.
\end{proof}

\begin{lemma}\label{tiny_PPYP}
  $\PP \hat Y \PA\T = 0$, $\hat Y \PA\T = \PA\T\PA \hat Y \PA\T$ and $\hat Y\PP\T = \PP\T\PP \hat Y \PP\T$.
\end{lemma}
\begin{proof}
  The unitary matrix $P = \begin{pmatrix}\PA \\ \PP\end{pmatrix}$ diagonalizes $\hat Y$ (so $P\hat Y P\T$ is diagonal), and $\PP \hat Y \PA\T$ extracts an off-diagonal block.
\end{proof}

\begin{lemma}\label{tiny_PNPJ}
  $\norm{v}^2 = \norm{\PA v}^2+\norm{\PN v}^2+\norm{\PA J v}^2$.
\end{lemma}
\begin{proof}
  The matrix $P = \begin{pmatrix}\PA \\ \PN \\ \PA J\end{pmatrix}$ is unitary (the rows are eigenvectors of $\hat Y$).
\end{proof}

\begin{lemma}\label{tiny_tWrW}
  $\sigma_{r+1}(W) \le \norm{\tilde W}$.
\end{lemma}
\begin{proof}
  $\tilde W = W(I-Q)$ for the projection $Q$ of rank $r$.
\end{proof}

\section{Bounding norms using top singular pairs}\label{s:tech}
Our proofs frequently want to bound norms of matrices $\norm{X}$ by considering their evolution $\frac{d}{dt}\norm{X}$. However, $\frac{d}{dt}\norm{X}$ might not be well-defined when the largest singular value is not unique. To circumvent these problems, we instead work with $\frac{d}{dt}u\T X v$, where $u,v$ is a top singular pair of $X$. This approach allows our proofs to work as if $\frac{d}{dt}\norm{X}$ existed. However, the approach requires a few general technical lemmas, which are proved below.

\begin{lemma}\label{tech0}
  Let $b > 0$, $X \colon [0,b] \to \R^{d\times d}$ be symmetric at all times, and right differentiable at time $0$. Assume $X_0$ is PSD (Positive Semi-Definite). Furthermore, for each $v \in \R^d$ such that $v\T X_0 v = 0$, we have $v\T \dot X_0 v > 0$. Then there exists $t \in (0,b]$ such that $X_\tau$ is PSD for all $\tau \in [0,t]$.
\end{lemma}
\begin{proof}
  Let $\mathcal{N}$ and $\mathcal{N}^\perp$ be the null space and row space of $X_0$. Consider the quadratic form $u \mapsto u\T \dot X_0 u$ with domain $\mathcal{N}$. By assumption, this quadratic form is positive definite, so there exists some $c_1 > 0$ such that $u\T \dot X_0 u \ge c_1 \norm{u}^2$ for $u \in \mathcal{N}$. Since $X_0$ is PSD, there also exists some $c_2 > 0$ such that $v\T X_0 v \ge c_2 \norm{v}^2$ for $v \in \mathcal{N}^\perp$. Next, since $X$ is right differentiable at time 0, we have $X_\tau = X_0+\tau \dot X_0+o(\tau)$. Here the little-o $o(\tau)$ denotes terms which asymptotically go to zero faster than $\tau$ as $\tau \to 0$.

  We may decompose an arbitrary vector as $u+v$ for $u \in \mathcal{N}$ and $v \in \mathcal{N}^\perp$. Furthermore,
  \begin{align*}
    &(u+v)\T X_\tau (u+v)\\
    &= (u+v)\T \left(X_0+\tau \dot X_0+o(\tau)\right) (u+v)\\
                         &= v\T X_0 v + \tau u\T\dot X_0 u + 2\tau u\T\dot X_0 v + \tau v\T \dot X_0 v + o(\tau)\left(\norm{u}^2+\norm{v}^2\right)\\
                         &\ge c_2\norm{v}^2 + \tau c_1 \norm{u}^2 - \tau\norm{\dot X_0}\left(\norm{v}^2+2\norm{u}\norm{v}\right)+o(\tau)\left(\norm{u}^2+\norm{v}^2\right).
  \end{align*}
  By definition of little-o, there exists some $t > 0$, independent of $u,v$, such that when $\tau \le t$, we have
  \[-o(\tau)\left(\norm{u}^2+\norm{v}^2\right) \le \frac{1}{2}(\tau c_1 \norm{u}^2+c_2\norm{v}^2).\] Similarly, there exists $t > 0$ small enough that $\tau\norm{\dot X_0}\left(\norm{v}^2+2\norm{u}\norm{v}\right) \le \frac{1}{2}(\tau c_1 \norm{u}^2+c_2\norm{v}^2)$ whenever $\tau \le t$. Thus, we can find $t > 0$ such that $(u+v)\T X_\tau (u+v) \ge 0$ for $\tau \in [0,t]$, which concludes the proof.
\end{proof}

\begin{corollary}\label{tech_eig}
  Let $b > a$, $X \colon [a,b] \to \R^{d\times d}$ be symmetric at all times, and right differentiable at time $a$. Assume $c \in \R$, $\lambda_1(X_a) \le c$. Furthermore, assume that whenever $\norm{v} = 1, v\T X_a v = c$ for some $v \in \R^d$, we also have $v\T \dot X_a v < 0$. Then there exists $t \in (a,b]$ such that $\lambda_1(X_\tau) \le c$ for all $\tau \in [a,t]$.
\end{corollary}
\begin{proof}
  Apply \Cref{tech0} to $\tilde X_t = cI-X_{t+a}$.
\end{proof}

\begin{corollary}\label{tech_norm}
  Let $b > a$, $X \colon [a,b] \to \R^{d_1\times d_2}$ be right differentiable at time $a$. Furthermore, assume $c > 0$, $\norm{X_a} \le c$. Additionally, assume either $\norm{X_a} < c$, or all top singular pairs $u,v$ of $X_a$ satisfy $u\T \dot X_a v < 0$. Then there exists $t \in (a,b]$ such that $\norm{X_\tau} \le c$ for all $\tau \in [a,t]$.
\end{corollary}
\begin{proof}
  If $\norm{X_a} < c$, there exists such $t$ by right differentiability of $X$ at time $a$. Otherwise, apply \Cref{tech_eig} with $\tilde X_t = X_t\T X_t$ and $\tilde c = c^2$. We verify the requirements of \Cref{tech_eig}. First, $\lambda_1(\tilde X_a) = \norm{X_a}^2 = c^2 = \tilde c$. Second, if $v \in \R^{d_2}$ satisfies $\norm{v} = 1, v\T\tilde X_a v = \tilde c$ then $\norm{X_a v}^2 = c^2$. Together with $\norm{X_a} \le c$, this means $(\frac{1}{c}X_a v, v)$ is a top singular pair of $X_a$. By assumption, we hence have $0 > \frac{1}{c}v\T X_a \dot X_a v = \frac{1}{2c}v\T \dot{\tilde X_a} v$.
\end{proof}

\section{Proof of \texorpdfstring{\Cref{rip_theorem}}{}}\label{s:rip_proof}
The proof plan is to first prove that \Cref{asmp:pert} from \Cref{s:main_results} holds under the assumptions in \Cref{rip_theorem}. This will let us use \Cref{master_theorem} to prove \Cref{rip_theorem}. Let us first recall the relevant assumptions from \Cref{rip_theorem}.

Let $U \colon \R_{\ge0} \to \R^{n_1\times h}$ and $V \colon \R_{\ge0} \to \R^{n_2\times h}$ follow perturbed gradient descent
\begin{equation}\label{eq:perturbed_gd3}
  \begin{split}
    \begin{pmatrix}U_{(k+1)\eta}\\V_{(k+1)\eta}\end{pmatrix} &= (I+\eta\tilde R_{k\eta})\begin{pmatrix}U_{k\eta}\\V_{k\eta}\end{pmatrix}, \text{where}\\
    \tilde R_{k\eta} &= \begin{pmatrix} 0 & (\A_k^*\A_k)(Y-U_{k\eta}V_{k\eta}\T)\\(\A_k^*\A_k)(Y-U_{k\eta}V_{k\eta}\T)\T & 0\end{pmatrix} + \tilde E_k,
  \end{split}
\end{equation}
with learning rate 
\[\eta \le \frac{c_2}{\kappa^3 \norm{Y}\log(\sigma_r(Y)/\epsilon^2)}\]
and perturbations $\{\tilde E_k\}_{k\ge0}$ with \[\norm{\tilde E_k} \le \xi \le \frac{c_4}{\kappa^3 \log(\sigma_r(Y)/\epsilon^2)}\norm{Y}.\]
Also, assume that the measurement operators $\A_k$ satisfy RIP (\Cref{def:rip} from \Cref{s:rip_result}) with rank $r+1$ and constant \[\rho \le \frac{c_3}{\sqrt r\kappa^2}.\]

We would like to show that \Cref{asmp:pert} holds for appropriate constants $\eta, \gamma, \beta, \mu$, and $\nu$. That is, we would like to find a piecewise smooth perturbation $E$ such that the following flow interpolates the perturbed gradient descent iterates:

\begin{align}
  \dt W_t = (R_t+E_t) W_t,
\end{align}
where $W_t = \begin{pmatrix}U_t\\V_t\end{pmatrix}$ and $R_t = \begin{pmatrix}0 & Y-U_tV_t\T \\ Y\T-V_tU_t\T & 0\end{pmatrix}$.

The perturbation should satisfy the following bound for all $t \ge 0$. Let $\tau = \lfloor t / \eta \rfloor \eta$. If 
\begin{align}\label{eq:pert_asmp}
  \norm{W_\tau} \le \frac{3}{2}\sqrt{\norm{Y}}\quad \text{ and }\quad \gamma \sigma_{r+1}^2(W_\tau) \le \norm{Y},
\end{align}
then
\begin{align*}
  \norm{E_t} &\le \beta \left(\norm{R_\tau}+\gamma \sigma_{r+1}^2(W_\tau)\right) + \mu\norm{Y}\\
  \norm{J E_t + E_t\T J} &\le \nu \norm{Y},
\end{align*}
where $J = \begin{pmatrix}I_{n_1} & 0\\0 & -I_{n_2}\end{pmatrix}$.

These results are essentially captured by the following theorem.
\begin{theorem}\label{rip_reduction}
Let $U$ and $V$ follow the gradient descent dynamics \eqref{eq:perturbed_gd3} with $\eta \le 1/(12\norm{Y})$, and assume for all $k\ge 0$ that $\norm{\tilde E_k} \le \tilde\mu \norm{Y}$ with $\tilde\mu \le 1/2$. Furthermore, assume the measurement operators $\A_k$ satisfy RIP with rank $r+1$ and constant $\rho \le 1 / (16\sqrt r)$.

Then there exists a piecewise smooth perturbation $E$ with the following properties. The gradient flow perturbed by $E$, has a solution $U_t,V_t$ which interpolates the gradient descent iterates $U_{k\eta},V_{k\eta}$ at the points $t = \eta k$. Furthermore, $E$ satisfies \Cref{asmp:pert} with $\eta$, $\gamma = \frac{\min(n_1,n_2)}{r} \ge 1$, $\beta = 10\eta\norm{Y} + 6\sqrt r \rho$, $\mu = 2\tilde\mu$, $\nu = 20\eta\norm{Y}+2\tilde\mu$, and $T = \infty$.
\end{theorem}
\begin{proof}
  Let us first rewrite \eqref{eq:perturbed_gd3} by splitting $\tilde R_{k\eta}$ into two pieces as follows:
  \begin{align*}
    W_{(k+1)\eta} = (I + \eta R_{k\eta} + \eta \tilde E_k)W_{k\eta}, 
  \end{align*}
  where $W = \begin{pmatrix}U\\V\end{pmatrix}$, $R = \begin{pmatrix}0 & Y-UV\T\\Y\T-VU\T\end{pmatrix}$ and $\hat E_k = \begin{pmatrix}0 & E^\A_k\\(E^\A_k)\T & 0\end{pmatrix} + \tilde E_k$, where $E^\A_k = (\A_k^*\A_k)(Y-U_{k\eta}V_{k\eta}\T) - (Y-U_{k\eta}V_{k\eta}\T)$. Note $\norm{\tilde E_k} \le \norm{E^\A_k} + \tilde\mu \norm{Y}$.

  Next, note that $J\begin{pmatrix}0 & E^\A_k\\(E^\A_k)\T & 0\end{pmatrix} + \begin{pmatrix}0 & E^\A_k\\(E^\A_k)\T & 0\end{pmatrix}J = 0$, where $J = \begin{pmatrix}I_{n_1} & 0\\0 & -I_{n_2}\end{pmatrix}$. Hence, $\norm{J\tilde E_k + \tilde E_k\T J} \le 2\tilde\mu \norm{Y}$.

  Now, \Cref{EA_bound} and $\gamma \coloneqq \frac{\min(n_1,n_2)}{r} \ge 1$ give $\norm{E^\A_k} \le 2\sqrt{r}\rho\left(\norm{R_{k\eta}} + \frac{3}{2}\gamma\sigma_{r+1}^2(W_{k\eta})\right)$. We may assume \eqref{eq:pert_asmp}. Then by \Cref{tiny_R_bound}, we have $\norm{R_{k\eta}} \le \frac{17}{8}\norm{Y}$. Also, $\gamma \sigma_{r+1}^2(W_{k\eta}) \le \norm{Y}$ and $\rho \le 1/(16\sqrt r)$. Thus, $\norm{E^\A_k} \le \frac{1}{2}\norm{Y}$. Together with the assumption $\tilde\mu \le 1/2$, this implies $\norm{\tilde E_k} \le \norm{Y}$.

  We may therefore apply \Cref{prop:discrete}, which yields the bounds
  \begin{align*}
  \norm{E_t} &\le 10\eta \norm{Y}\norm{R_{k\eta}} + 4\sqrt{r}\rho\left(\norm{R_{k\eta}} + \frac{3}{2}\gamma\sigma_{r+1}^2(W_{k\eta})\right) + 2\tilde\mu\norm{Y}\\
    \norm{JE_t + E_t\T J} &\le 20\eta\norm{Y}^2 + 2\tilde\mu \norm{Y}.
  \end{align*}

  Hence, $E$ satisfies \Cref{asmp:pert} with $\eta$, $\gamma = \frac{\min(n_1,n_2)}{r} \ge 1$, $\beta = 10\eta\norm{Y} + 6\sqrt r \rho$, $\mu = 2\tilde\mu$ and $\nu = 20\eta\norm{Y}+2\tilde\mu$.
\end{proof}

We are now ready to prove \Cref{rip_theorem} using \Cref{master_theorem,rip_reduction}.
\begin{proof}[Proof of \Cref{rip_theorem}]
  Recall the assumptions, with $\gamma \coloneqq \frac{\min(n_1,n_2)}{r}$,  
  \begin{align*}
    \rho &\le \frac{c_1}{\sqrt r\kappa^2}\\
    \epsilon &\le \frac{c_2\sqrt{\sigma_r(Y)}}{\gamma^2+\kappa}\\
    \eta &\le \frac{c_3}{\kappa^3 \norm{Y}\log(\sigma_r(Y)/\epsilon^2)}\\
    \xi &\le \frac{c_4}{\kappa^3 \log(\sigma_r(Y)/\epsilon^2)}\norm{Y}\\
    \alpha &\le C_1.
  \end{align*}
  Note that the bound on $\epsilon \le \frac{c_2\sqrt{\sigma_r(Y)}}{\gamma^2+\kappa}$ implies
  \begin{align}\label{eq:frac_bound}
    \frac{\epsilon^2}{\sigma_r(Y)} \le \frac{c_2^2}{(\gamma^2+\kappa)^2} < \frac{1}{4}c_2^2.
  \end{align}
  Assume $c_2 \le 1$. Then $\eta < c_3/\sigma_r(Y)$. Also, let \[T \coloneqq \frac{5}{\sigma_r(Y)}\log\left(\frac{\sigma_r(Y)}{\epsilon^2}\right).\] Note that \eqref{eq:frac_bound} implies $T \ge 5/\sigma_r(Y)$.

  Then \Cref{rip_reduction} shows that $E$ satisfies with $\eta$, $\gamma = \min(n_1,n_2)/r$,
  \begin{align*}
    \beta &= 10\eta\norm{Y} + 6\sqrt r \rho \le \frac{C_1(10c_3+6c_1)}{\alpha\kappa^2},\\
    \mu &= \frac{2\xi}{\norm{Y}} \le \frac{2C_1 c_4}{\alpha\kappa^2},\\
    \nu &= 20\eta\norm{Y}+\frac{2\xi}{\norm{Y}} \le \frac{C_1(100c_3+10c_4)}{\alpha\kappa^2 T\norm{Y}}.
  \end{align*}
  Fixing $\theta = 1/12$, we can hence pick $c_1,c_2,c_3,c_4 > 0$ small enough such that the assumptions of \Cref{master_theorem} hold. \Cref{master_theorem} therefore yields
  \begin{align*}
    \norm{R_T} &\le \tilde C_1\left(\beta\kappa\gamma+\sqrt{\kappa}\right)\left(\frac{\epsilon^2}{\sigma_r(Y)}\right)^{3/4}\sigma_r(Y) + \tilde C_2 \kappa \mu\norm{Y}\\
               &\le \tilde C_1\left(\gamma+\sqrt{\kappa}\right)\left(\frac{\epsilon^2}{\sigma_r(Y)}\right)^{11/12}\sigma_r(Y) + 2\tilde C_2 \kappa \xi.
  \end{align*}

  Recall $\frac{\epsilon^2}{\sigma_r(Y)} \le \frac{c_2^2}{(\gamma^2+\kappa)^2}$ which implies $(\gamma+\sqrt\kappa)\left(\frac{\epsilon^2}{\sigma_r(Y)}\right)^{1/4} \le \sqrt{2 c_2}$. Also note that $\norm{R_T} = \norm{Y-U_TV_T\T}$ by \Cref{tiny_dilation}. This means we can rewrite the bound in the desired form
  \begin{align*}
    \norm{Y-U_TV_T\T} = \norm{R_T} &\lesssim \left(\frac{\epsilon^2}{\sigma_r(Y)}\right)^{2/3}\sigma_r(Y)+\kappa\xi.
  \end{align*}
\end{proof}

\section{Proof of \texorpdfstring{\Cref{sgd_theorem}}{}}\label{s:sgd_proof}
The assumptions and proof of \Cref{sgd_theorem} are similar to those of \Cref{rip_theorem}, which are detailed in \Cref{s:rip_proof}. The difference is that in the case of stochastic gradient descent, we do not assume that the measurement operators $\A_k$ satisfy RIP. Instead, we assume that $\A_k \colon \R^{n_1\times n_2} \to \R^m$ are independent and of the form 
\begin{align}\label{eq:sgd_form}
  [\A_k(X)]_i = \frac{1}{\sqrt m} \langle X, [A_k]_i\rangle,
\end{align}
for $i = 1,\dots,m$, where $[A_k]_i \in \R^{n_1\times n_2}$ has i.i.d. entries with distribution $\mathcal{N}(0,1)$. Since the algorithm is randomized, we can only guarantee success with a high probability $1-\delta$, for some failure probability $\delta \in (0,1)$. We assume the mini-batch size $m$ satisfies
\[m \ge C_2 (\log(n_1+n_2)+\log(K)+\log(1/\delta)) r(n_1+n_2) \kappa^4,\]
where $K$ is the number of iterations.

Analogous to \Cref{s:rip_proof}, we first prove that \Cref{asmp:pert} from \Cref{s:main_results} holds, and then apply \Cref{master_theorem}.
\begin{theorem}\label{sgd_reduction}
  Let $U$ and $V$ follow the gradient descent dynamics \eqref{eq:perturbed_gd3} with $\eta \le 1/(12\norm{Y})$, and assume for all $k\ge 0$ that $\norm{\tilde E_k} \le \tilde\mu \norm{Y}$ with $\tilde\mu \le 1/2$.

  Furthermore, assume the measurement operators $\A_k$ are of the form \eqref{eq:sgd_form}. Select $\delta \in (0,1)$ and $0 < \rho \le 1$, and assume
  \[m \ge \frac{C_2}{\rho^2}\left(\log(n_1+n_2)+\log(K)+\log(1/\delta)\right) r(n_1+n_2),\]
  where $C_2$ is a universal constant.

  Then with probability at least $1-\delta$, there exists a piecewise smooth perturbation $E$ with the following properties. The gradient flow perturbed by $E$, has a solution $U_t,V_t$ which interpolates the gradient descent iterates $U_{k\eta},V_{k\eta}$ at the points $t = \eta k$, for $K \ge 1$ iterations $0 \le k \le K$. Furthermore, $E$ satisfies \Cref{asmp:pert} with $\eta$, $\gamma = \sqrt{\min(n_1,n_2)/r}$, $\beta = 10\eta\norm{Y} + \rho$, $\mu = 2\tilde\mu$, $\nu = 20\eta\norm{Y}+2\tilde\mu$, and $T = K\eta$.
\end{theorem}
\begin{proof}
  The first part of the proof is completely analogous to the proof of \Cref{rip_reduction}. Hence, we have
  \begin{align}
    \norm{\tilde E_k} \le \norm{E^\A_k}+\tilde\mu\norm{Y}\label{eq:sgd_E_bound},\\
    \norm{J\tilde E_k+\tilde E_k\T J} \le 2\tilde\mu\norm{Y}\nonumber,
  \end{align}
  where $E^\A_k = (\A_k^*\A_k)(Y-U_{k\eta}V_{k\eta}\T)-(Y-U_{k\eta}V_{k\eta}\T)$ and $J = \begin{pmatrix}I_{n_1} & 0\\ 0 & -I_{n_2}\end{pmatrix}$.

  Note that our assumed form for $\A_k$ implies 
  \[E^\A_k = \frac{1}{m}\sum_{i=1}^m(\langle A_i, Y-U_{k\eta}V_{k\eta}\T\rangle A_i)-(Y-U_{k\eta}V_{k\eta}\T).\]
  This is exactly the form expected by \Cref{EA_bound_sgd}, which yields the following: Let $\xi \le c_1\sqrt m$ and $\gamma \coloneqq \sqrt{\min(n_1,n_2)/r}$, then
  \begin{align*}
    \norm{E^\A_k} \le \xi\sqrt{\frac{r(n_1+n_2)}{m}}\left(\sqrt 2\norm{R_{k\eta}}+\sqrt{3}\gamma\sigma_{r+1}^2(W_{k\eta})\right),
  \end{align*}
  with probability at least $1-(n_1+n_2)e^{-c_2\xi^2}$. By the union bound, it holds for all $0 \le k < K$ with probability at least $1-K(n_1+n_2)e^{-c_2\xi^2}$.

  Let $\delta = K(n_1+n_2)e^{-c_2\xi^2}$, such that $\xi^2 = c_2^{-1}\left(\log(n_1+n_2)+\log(K)+\log(1/\delta)\right)$. Then the bound on $m$ reads $m \ge \frac{c_2 C_2}{\rho^2}\xi^2 r(n_1+n_2)$. Hence,
  \begin{align}\label{eq:EAk_bound}
    \norm{E^\A_k} \le \frac{\rho}{\sqrt{c_2 C_2}} \left(\sqrt 2\norm{R_{k\eta}}+\sqrt{3}\gamma\sigma_{r+1}^2(W_{k\eta})\right).
  \end{align}

  For a fixed $0 \le k < K$, we may assume \eqref{eq:pert_asmp}. Then by \Cref{tiny_R_bound}, we have $\norm{R_{k\eta}} \le \frac{17}{8}\norm{Y}$. Also, $\gamma \sigma_{r+1}^2(W_{k\eta}) \le \norm{Y}$. Thus, $\sqrt 2\norm{R_{k\eta}}+\sqrt{3}\gamma\sigma_{r+1}^2(W_{k\eta}) \le 5\norm{Y}$. If $\rho \le \frac{1}{10}\sqrt{c_2 C_2}$, inserting into \eqref{eq:EAk_bound} implies $\norm{E^\A_k} \le \frac{1}{2}\norm{Y}$. Furthermore, since $\tilde \mu \le 1/2$, we also have $\norm{\tilde E_k} \le \norm{Y}$ by \eqref{eq:sgd_E_bound}.

  We may therefore apply \Cref{prop:discrete}, which yields the bounds
  \begin{align*}
    \norm{E_t} &\le 10\eta \norm{Y}\norm{R_{k\eta}} + \frac{2\rho}{\sqrt{c_2 C_2}} \left(\sqrt 2\norm{R_{k\eta}} + \sqrt 3\gamma\sigma_{r+1}^2(W_{k\eta})\right) + 2\tilde\mu\norm{Y},\\
    \norm{JE_t + E_t\T J} &\le 20\eta\norm{Y}^2 + 2\tilde\mu \norm{Y}.
  \end{align*}
  We select $C_2 = \max\left(100, 1/c_1^2\right)/c_2$, such that both the requirements $\xi \le c_1\sqrt m$ and $\rho \le \frac{1}{10}\sqrt{c_2 C_2}$ are satisfied since $\rho \le 1$. Then, with probability at least $1-\delta$, $E$ satisfies \Cref{asmp:pert} with $\eta$, $\gamma = \sqrt{\min(n_1,n_2)/r}$, $\beta = 10\eta\norm{Y} + \rho$, $\mu = 2\tilde\mu$ and $\nu = 20\eta\norm{Y}+2\tilde\mu$.
\end{proof}

We are now ready to prove \Cref{sgd_theorem} using \Cref{master_theorem,sgd_reduction}.
\begin{proof}[Proof of \Cref{sgd_theorem}]
  The proof is very similar to the proof of \Cref{rip_theorem}. Recall the assumptions 
  \begin{align*}
    \epsilon &\le \frac{c_1\sqrt{\sigma_r(Y)}}{\gamma^2+\kappa}\\
    \eta &\le \frac{c_2}{\kappa^3 \norm{Y}\log(\sigma_r(Y)/\epsilon^2)}\\
    m &\ge C_2\frac{\kappa^4}{c_3^2}\left(\log(n_1+n_2)+\log(K)+\log(1/\delta)\right) r(n_1+n_2)\\
    \xi &\le \frac{c_4}{\kappa^3 \log(\sigma_r(Y)/\epsilon^2)}\\
    \alpha &\le C_1.
  \end{align*}

  Like in the proof of \Cref{rip_theorem}, assume $c_1 \le 1$, which leads to $\log(\sigma_r(Y)/\epsilon^2) \ge 1$ and \[T \coloneqq \frac{5}{\sigma_r(Y)}\log\left(\frac{\sigma_r(Y)}{\epsilon^2}\right) \ge \frac{5}{\sigma_r(Y)}.\]

  Using \Cref{sgd_reduction} in place of \Cref{rip_reduction}, we get similar bounds to \Cref{rip_theorem}. Specifically, with probability at least $1-\delta$, $E$ satisfies \Cref{asmp:pert} with $\eta$, $\gamma = \sqrt{\min(n_1,n_2)/r}$,
  \begin{align*}
    \beta &= 10\eta\norm{Y} + \frac{c_3}{\kappa^2} \le \frac{C_1(10c_2+c_3)}{\alpha\kappa^2},\\
    \mu &= \frac{2\xi}{\norm{Y}} \le \frac{2C_1 c_4}{\alpha\kappa^2},\\
    \nu &= 20\eta\norm{Y}+\frac{2\xi}{\norm{Y}} \le \frac{C_1(100c_2+10c_4)}{\alpha\kappa^2 T\norm{Y}}.
  \end{align*}
  Fixing $\theta = 1/12$, we can hence pick $c_1,c_2,c_3,c_4 > 0$ small enough such that the assumptions of \Cref{master_theorem} hold. The rest of the proof is identical to that in the proof of \Cref{rip_theorem}.
\end{proof}

\section{Reduction from perturbed gradient descent to perturbed gradient flow}\label{s:discrete}
The following proposition bounds the perturbation corresponding to the discretization error of learning rates $\eta > 0$. It is used to prove \Cref{rip_reduction,sgd_reduction}.

Recall $R_t = \hat Y - \frac{1}{2}\left(W_tW_t\T-JW_tW_t\T J\right)$, where $J = \begin{pmatrix}I_{n_1} & 0\\ 0 & -I_{n_2}\end{pmatrix}$ and $\hat Y = \begin{pmatrix}0 & Y\\Y\T & 0\end{pmatrix}$.

\begin{proposition}\label{prop:discrete}
  Assume $\eta \le 1/(12\norm{Y})$, $\norm{W_{k\eta}} \le \frac{3}{2}\sqrt{\norm{Y}}$ and $\norm{\tilde E_k} \le \norm{Y}$. Let $\dot W_t = (R_t + E_t) W_t$, where $E_t = \frac{1}{\eta}\log\left(I+\eta R_{k\eta}+\eta \tilde E_k\right) - R_t$, $t \ge 0$ and $k = \left\lfloor \frac{t}{\eta} \right\rfloor$. Then
  \begin{align}\label{eq:W_kp1}
    W_{(k+1)\eta} = (I+\eta R_{k\eta}+\eta \tilde E_k) W_{k\eta}.
  \end{align}

  Furthermore,
  \begin{align*}
    \norm{E_t} &\le 10\eta\norm{Y} \norm{R_{k\eta}} + 2\norm{\tilde E_k},\\
    \norm{J E_t + E_t\T J} &\le 20\eta\norm{Y}^2+\norm{J\tilde E_k + \tilde E_k\T J}.
  \end{align*}
\end{proposition}
\begin{proof}
  Define $\tilde R_{k\eta} = R_{k\eta} + \tilde E_k$. Note that $\norm{W_{k\eta}} \le \frac{3}{2}\sqrt{\norm{Y}}$ implies $\norm{R_{k\eta}} \le \frac{17}{8}\norm{Y}$ by \Cref{tiny_R_bound}. Because $\eta \le \frac{1}{12\norm{Y}}$, we have $\eta\norm{R_{k\eta}} \le \frac{1}{3}$. Similarly, $\norm{\tilde E_k} \le \norm{Y}$ implies $\eta\norm{\tilde E_k} \le \frac{1}{12}$. Together, these imply $\eta\norm{\tilde R_{k\eta}} \le \frac{2}{3}$.

  Note that the matrix logarithm $\log\left(I+\eta \tilde R_{k\eta}\right)$ is well-defined since $\eta \norm{\tilde R_{k\eta}} \le \frac{2}{3} < 1$ by the triangle inequality.

  Simplify $\dot W_t = (R_t + E_t) W_t = \frac{1}{\eta}\log\left(I+\eta \tilde R_{k\eta}\right) W_t$, which has the analytical solution $W_t = \left(I+\eta \tilde R_{k\eta}\right)^{\frac{t}{\eta}-k} W_{k\eta}$ for $t \in [k\eta,(k+1)\eta]$. Inserting $t = (k+1)\eta$ gives \eqref{eq:W_kp1}.

  Next, by the definition $E_t = \frac{1}{\eta}\log\left(I+\eta \tilde R_{k\eta}\right) - R_t$, and the triangle inequality, we have
  \begin{align}\label{E_triangle}
    \norm{E_t} &= \norm{\frac{1}{\eta}\log(I+\eta\tilde R_{k\eta})-R_t}\\
    &\le \norm{\frac{1}{\eta}\log\left(I+\eta\tilde R_{k\eta}\right)-\tilde R_{k\eta}}+\norm{\tilde R_{k\eta}-R_{k\eta}}+\norm{R_{k\eta}-R_t}.
  \end{align}
  The second term is equal to $\norm{\tilde E_k}$ by the definition of $\tilde R_{k\eta}$.

  Next, define $S_t \coloneqq (I+\eta\tilde R_{k\eta})^{\frac{t}{\eta}-k}-I$. Then $W_t = (I+S_t)W_{k\eta}$ by construction and $\norm{S_t} \le \eta\norm{\tilde R_{k\eta}} \le \frac{2}{3}$. Now,
  \begin{align*}
    \norm{R_{k\eta}-R_t} &= \frac{1}{2}\norm{W_{k\eta}W_{k\eta}\T-W_tW_t\T - J(W_{k\eta}W_{k\eta}\T-W_tW_t\T)J}\\
                         &\le \norm{W_{k\eta}W_{k\eta}\T-W_tW_t\T} = \norm{W_{k\eta}W_{k\eta}\T-(I+S_t)W_{k\eta}W_{k\eta}\T(I+S_t)}\\
                         &\le \norm{W_{k\eta}}^2\left(2\norm{S_t}+\norm{S_t}^2\right) \le \frac{8}{3}\eta\norm{\tilde R_{k\eta}}\norm{W_{k\eta}}^2.
  \end{align*}
  The first equality expands the definition of $R_t$ and cancels the $\hat Y$ terms. The next inequality uses \Cref{tiny_XmJXJ}. Then, $W_t = (I+S_t)W_{k\eta}$ is used. We then use the triangle inequality, and bound products by products of norms. Finally, use $\norm{S_t} \le \eta\norm{\tilde R_{k\eta}} \le \frac{2}{3}$.

  Since $\eta\norm{\tilde R_{k\eta}} \le \frac{2}{3}$, we may use a matrix generalization of the inequality $|\log(1+x)-x| \le x^2$ for $x \in [-2/3,2/3]$. Specifically, using the power series for matrix logarithms, we have for any matrix $B$ with $\norm{B} \le \frac{2}{3}$, that 
  \begin{align*}
    \norm{B-\log(I+B)} &= \norm{B-\sum_{i=1}^\infty (-1)^{i+1} \frac{1}{i} B^i} \le \norm{\sum_{i=2}^\infty (-1)^{i+1} \frac{1}{i} B^i}\\
                       &\le \sum_{i=2}^\infty \frac{1}{i} \norm{B}^i \le \norm{B}^2 \sum_{i=2}^\infty \frac{1}{i} \left(\frac{2}{3}\right)^{i-2}\\
                       &= \frac{3\log(27)-6}{4}\norm{B}^2 \le \norm{B}^2.
  \end{align*}

  Hence, we obtain the following bound on the first term of \eqref{E_triangle}  
  \[ \norm{\frac{1}{\eta}\log(I+\eta\tilde R_{k\eta})-\tilde R_{k\eta}} \le \eta\norm{\tilde R_{k\eta}}^2.\]

  Combining the preceeding bounds on the three terms of \eqref{E_triangle}, we get
  \begin{align*}
    \norm{E_t} &\le \eta\norm{\tilde R_{k\eta}}^2 + \norm{\tilde E_k} + \frac{8}{3}\eta\norm{\tilde R_{k\eta}}\norm{W_{k\eta}}^2.
  \end{align*}

  Recall the bounds from the top of this proof, $\norm{W_{k\eta}} \le \frac{3}{2}\sqrt{\norm{Y}}$, $\norm{R_{k\eta}} \le \frac{17}{8}\norm{Y}$, $\eta \le \frac{1}{12\norm{Y}}$, $\norm{\tilde E_k} \le \norm{Y}$. The triangle inequality gives $\norm{\tilde R_{k\eta}} \le \frac{25}{8}\norm{Y}$ by the definition $\tilde R_{k\eta} = R_{k\eta} + \tilde E_k$.

  These let us bound
  \begin{align*}
    \norm{E_t} &\le \eta\norm{\tilde R_{k\eta}}^2 + \norm{\tilde E_k} + \frac{8}{3}\eta\norm{\tilde R_{k\eta}}\norm{W_{k\eta}}^2\\
               &\le \left(\frac{25}{8}+6\right)\eta\norm{Y}\norm{\tilde R_{k\eta}} + \norm{\tilde E_k}\\
               &\le 10\eta\norm{Y}\norm{R_{k\eta}} + 2\norm{\tilde E_k}.
  \end{align*}

  Next, we bound $\norm{JE_t+E_t\T J}$ similarly to \eqref{E_triangle}. However, note that $\norm{J X + X\T J} \le 2\norm{X}$ since $\norm{J} = 1$, and that $J R_t + R_t\T J = J R_{k\eta} + R_{k\eta}\T J = 0$. This simplifies
  \begin{align*}
    \norm{JE_t + E_t\T J} &\le 2\norm{\frac{1}{\eta}\log\left(I+\eta\tilde R_{k\eta}\right)-\tilde R_{k\eta}}+\norm{J\tilde E_k + \tilde E_k\T J}\\
                          &\le 2\eta\norm{\tilde R_{k\eta}}^2+\norm{J\tilde E_k + \tilde E_k\T J}\\
                          &\le 20\eta\norm{Y}^2+\norm{J\tilde E_k + \tilde E_k\T J}.
  \end{align*}
\end{proof}

\section{Proofs for gradient descent with RIP measurements}\label{s:rip_proof_appendix}
The following bound is used in the proof of \Cref{rip_reduction} to bound the error stemming from an imperfect measurement operator $\A$. Note that the error is zero for a perfect observation operator $\rho = 0$.

\begin{proposition}\label{EA_bound}
  Let $\bar R = Y-U V\T$, $R = \begin{pmatrix}0 & \bar R\\\bar R\T & 0\end{pmatrix}$, $W = \begin{pmatrix}U\\V\end{pmatrix}$ and $E^\A = (\A^*\A)(\bar R) - \bar R$. Assume the linear operator $\A \colon \R^{n_1\times n_2} \to \R^m$ satisfies RIP (\Cref{def:rip} from \Cref{s:rip_result}) with rank $r+1$ and constant $\rho$. Then
  \[\norm{E^\A} \le 2\sqrt{r}\rho\left(\norm{R} + \left(\frac{\min(n_1,n_2)}{2r}+1\right)\sigma_{r+1}^2(W)\right)\]
\end{proposition}
\begin{proof}
  Let $Q \in \R^{h\times h}$ be a projection such that $\rank(Q) = r$ and $\norm{WQ_\perp} = \sigma_{r+1}(W)$ where $Q_\perp \coloneqq I-Q$. Recall $\rank(Y) = r$ and note $\rank(UQV\T) \le \rank(Q) = r$, hence $\rank(Y-UQV\T) \le 2r$. 

  By \Cref{lem:RIP}, we have $\norm{(\A^*\A)(X)-X} \le \rho\norm{X}_F$ for $X \in \R^{n_1\times n_2}$ with $\rank(X) \le r$. For a matrix $X \in \R^{n_1\times n_2}$ with $\rank(X) > r$, we may split it into a sum $X = \sum_{i=1}^{\left\lceil \frac{\rank(X)}{r} \right\rceil} X_i$ where addend $X_i$ satisfies $\rank(X_i) \le r$ and $\norm{X_i} \le \norm{X}$. Hence, 
  \begin{align*}
    \norm{(\A^*\A)(X)-X} &\le \sum_{i=1}^{\left\lceil \frac{\rank(X)}{r} \right\rceil} \norm{(\A^*\A)(X_i)-X_i}\\
                         &\le \rho \sum_{i=1}^{\left\lceil \frac{\rank(X)}{r} \right\rceil} \norm{X_i}_F \le \rho \left\lceil \frac{\rank(X)}{r} \right\rceil \sqrt{r}\norm{X}.
  \end{align*}
  The second inequality uses \Cref{lem:RIP}. The third uses $\norm{X_i}_F \le \sqrt{\rank(X_i)}\norm{X_i} \le \sqrt r\norm{X}$.

  Inserting $Y-UQV\T$ and $UQ_\perp V\T$ for $X$, we get 
  \begin{align*}
    \norm{(\A^*\A)(Y-UQV\T)-(Y-UQV\T)} &\le 2\sqrt r\rho\norm{Y-UQV\T},\\
    \norm{(\A^*\A)(UQ_\perp V\T)-UQ_\perp V\T} &\le \sqrt{r}\rho\left\lceil\frac{\min(n_1,n_2)-r}{r}\right\rceil \norm{UQ_\perp V\T}.
  \end{align*}

  We may then bound
  \begin{align*}
    \norm{E_\A} &= \norm{(\A^*\A)(Y-UV\T)-(Y-UV\T)}\\
                &\le \norm{(\A^*\A)(Y-UQV\T)-(Y-UQV\T)} + \norm{(\A^*\A)(UQ_\perp V\T)-UQ_\perp V\T}\\
                &\le 2\sqrt r\rho\norm{Y-UQV\T} + \sqrt{r}\rho\frac{\min(n_1,n_2)}{r} \norm{UQ_\perp V\T}\\
                &\le \sqrt r\rho\left(2\norm{Y-UV\T} + \left(2+\frac{\min(n_1,n_2)}{r}\right) \norm{UQ_\perp V\T}\right)\\
                &\le 2\sqrt r\rho\left(\norm{Y-UV\T} + \left(\frac{\min(n_1,n_2)}{2r}+1\right)\sigma_{r+1}^2(W)\right).
  \end{align*}
  The first equality uses the definition of $E_\A$. Then the triangle inequality is used. Next, we apply the bounds from above. For the third inequality, we pull out a term $\norm{UQ_\perp V\T}$ using the triangle inequality. Finally, we bound $\norm{UQ_\perp V\T} \le \norm{WQ_\perp}^2 = \sigma_{r+1}^2(W)$.
\end{proof}

A simple proof of the following useful lemma is included for completeness. However, very similar results can be found for example in \citet[Lemma 7.3]{rip_small_init}.
\begin{lemma}\label{lem:RIP}
  Assume the linear operator $\A \colon \R^{n_1\times n_2} \to \R^m$ satisfies RIP (\Cref{def:rip} from \Cref{s:rip_result}) with rank $r+1$ and constant $\rho$. Then all matrices $X \in \R^{n_1\times n_2}$ with $\rank(X) \le r$ satisfy
  \begin{align*}
    \norm{(\A^*\A)(X)-X} \le \rho\norm{X}_F.
  \end{align*}
\end{lemma}
\begin{proof}
  Assume without loss of generality that $\norm{X}_F = 1$. Let $u,v$ be a top singular pair of $(\A^*\A)(X)-X$. Then
  \begin{align*}
    \norm{(\A^*\A)(X)-X} &= \langle uv\T, (\A^*\A)(X)-X\rangle\\
                         &= \langle \A(uv\T), \A(X)\rangle-\langle uv\T,X\rangle\\
                         &= \frac{1}{4}\left(\norm{\A(uv\T+X)}_F^2-\norm{\A(uv\T-X)}_F^2\right)-\langle uv\T,X\rangle\\
                         &\le \frac{1}{4}\left((1+\rho)\norm{uv\T+X}_F^2-(1-\rho)\norm{uv\T-X}_F^2\right)-\langle uv\T,X\rangle\\
                         &= \frac{\rho}{4}\left(\norm{uv\T+X}_F^2+\norm{uv\T-X}_F^2\right)\\
                         &= \frac{\rho}{2}\left(\norm{uv\T}_F^2+\norm{X}_F^2\right) = \rho.
  \end{align*}
  The second equality uses the definition of the adjoint. The third equality uses the polarization identity. Next, use the RIP inequalities \eqref{eq:RIP} from \Cref{def:rip} in \Cref{s:rip_result}. The polarization identity then cancels the term $\langle uv\T,X\rangle$. Next, we use the parallelogram law. Finally, the identities $\norm{u} = \norm{v} = \norm{X}_F = 1$ are used.
\end{proof}

\section{Bounding measurement errors from random measurements}\label{s:sgd_proof_appendix}
This section is used to prove \Cref{EA_bound_sgd}, which replaces \Cref{EA_bound} in the stochastic gradient descent setting. The main difficulty will be proving \Cref{sgd_E_bound}, which essentially says that with high probability
\[\norm{E^\A} \lesssim \sqrt{\frac{n_1+n_2}{m}}\norm{Y-UV\T}_F.\]
We will combine that with the following bound on $\norm{Y-UV\T}_F$.

\begin{lemma}\label{RF_bound}
  Let $R = \begin{pmatrix}0 & Y-UV\T\\Y\T-VU\T & 0\end{pmatrix}$ and $W = \begin{pmatrix}U\\V\end{pmatrix}$. Then
  \begin{align*}
    \norm{Y-UV\T}_F \le \sqrt{r}\left(\sqrt 2\norm{R}+\sqrt{\frac{3\min(n_1,n_2)}{r}}\sigma_{r+1}^2(W)\right).
  \end{align*}
\end{lemma}
\begin{proof}
  Define $Q$ as in the proof of \Cref{EA_bound}, to get the same consequences. That is, we select the projection $Q \in \R^{h\times h}$ such that $\rank(Q) = r$, $WQ_\perp = \sigma_{r+1}(W)$ where $Q_\perp \coloneqq I-Q$. Also, $\rank(Y-UQV\T) \le 2r$ and $\rank(UQ_\perp V\T) \le \min(n_1,n_2)-r$.

  Hence, we have
  \begin{align*}
    \norm{Y-UV\T}_F &\le \norm{Y-UQV\T}_F+\norm{U Q_\perp V\T}_F\\
                    &\le \sqrt{2r}\norm{Y-UQV\T}+\sqrt{\min(n_1,n_2)-r}\norm{U Q_\perp V\T}\\
                    &\le \sqrt{2r}\norm{Y-UV\T}+(\sqrt{\min(n_1,n_2)-r}+\sqrt{2r})\norm{U Q_\perp V\T}\\
                    &\le \sqrt{2r}\norm{Y-UV\T}+\sqrt{3\min(n_1,n_2)}\norm{U Q_\perp V\T}\\
                    &\le \sqrt{2r}\norm{Y-UV\T}+\sqrt{3\min(n_1,n_2)}\sigma_{r+1}^2(W)\\
                    &\le \sqrt{r}\left(\sqrt 2\norm{R}+\sqrt{\frac{3\min(n_1,n_2)}{r}}\sigma_{r+1}^2(W)\right).
  \end{align*}
  The first inequality uses the triangle inequality. The second bound uses the bound $\norm{X}_F \le \sqrt{\rank(X)}\norm{X}$. Next, the triangle inequality is used again. Then, the Cauchy-Schwarz inequality for $(1,\sqrt 2)$ and $(\sqrt x,\sqrt y)$ gives $\sqrt x + \sqrt{2y} \le \sqrt{3(x+y)}$. The fifth inequality bounds $\norm{UQ_\perp V\T} \le \norm{WQ_\perp}^2 = \sigma_{r+1}^2(W)$. Finally, the identity $\norm{Y-UV\T} = \norm{R}$ from \Cref{tiny_dilation} is used.
\end{proof}

The rest of the section will be building up to proving \Cref{sgd_E_bound}, which will give a bound on $\norm{E^\A}$ used in \Cref{rip_reduction}, where $E^\A$ is of the form $E^\A = (\A^*\A)(X)-X$. Specifically, we will use it with $X = Y-UV\T$, but that is not important for the derivation. The term $(\A^*\A)(X)-X$ intuitively measures the distance between the actual measurements $(\A^*\A)(X)$, and perfect measurement $X$.

As assumed in \Cref{s:sgd_proof}, $\A \colon \R^{n_1\times n_2} \to \R^m$ has the following form
\begin{align*}
  [\A(X)]_i = \frac{1}{\sqrt m} \langle X, A_i\rangle,
\end{align*}
where for each $i = 1,\dots,m$, $A_i \in \R^{n_1\times n_2}$ has i.i.d. entries with distribution $\mathcal{N}(0,1)$.

This means $(\A^*\A)(X)-X = \frac{1}{m}\sum_{i=1}^m\left(\langle A_i, X\rangle A_i\right)-X$. The plan is to develop high probability bounds on $\norm{(\A^*\A)(X)-X}$ over the sampling of of $\{A_i\}_{i=1}^m$.

We will use the concept of sub-exponential variables and the sub-exponential norm from \citet[Definition 2.7.5]{vershynin2018high}.
\begin{definition}[Sub-exponential norm]
  A random variable $x$ is sub-exponential if it has finite sub-exponential norm. That is, if
  \begin{align*}
    \norm{x}_{\psi_1} \coloneqq \inf\{t > 0 : \E \exp(|x|/t) \le 2\} < \infty.
  \end{align*}

  For random matrices $X$ we write $\norm{X}_{\psi_1} \coloneqq \norm{(\norm{X})}_{\psi_1}$, and refer to this quantity as the sub-exponential norm of $X$.
\end{definition}

The sub-exponential norm is useful because the sub-exponential norm of $\langle A, X\rangle A-X$ is bounded. This essentially corresponds to the random measurement error from running stochastic gradient descent with a mini-batch size of one. We will then apply a matrix concentration tail bound for sub-exponential matrices to bound the error for larger mini-batch sizes.
\begin{proposition}\label{sub_exp_general}
  Fix an $n_1\times n_2$ matrix $X$ and let $A$ be an $n_1\times n_2$ random matrix with i.i.d. standard normal elements. Then the random matrix $Z = \langle A, X\rangle A-X$ is zero-mean and has sub-exponential norm at most $C\sqrt{n_1+n_2}\norm{X}_F$ for some universal constant $C$.
\end{proposition}
\begin{proof}[Proof of \Cref{sub_exp_general}]
  The zero-mean property $\E Z = 0$ follows from (the elements of) $A$ being normally distributed. Furthermore, note that \[\norm{Z} = \norm{\langle A, X\rangle A-X} \le \norm{\langle A, X\rangle A}+\norm{X} = |\langle A, X\rangle|\norm{A}+\norm{X}.\]
  The sub-exponential norm follows the triangle inequality \citep[Exercise 2.7.11, $\psi(x) = e^x-1$]{vershynin2018high}, and the sub-exponential norm of the constant $\norm{X}$ is $\norm{X}/\log(2)$, so it is enough to show that $|\langle A, X\rangle|\norm{A}$ has sub-exponential norm at most $C_1\sqrt{n_1+n_2}\norm{X}_F$. Since the sub-exponential norm of a product of sub-gaussian variables is less than the product of sub-gaussian norms of the terms \citep[Lemma 2.7.7]{vershynin2018high}, it is sufficient to show that $|\langle A, X\rangle|$ has sub-gaussian norm at most $C_2\norm{X}_F$, and that $\norm{A}$ has sub-gaussian norm at most $C_3\sqrt{n_1+n_2}$.

  Since $A$ is normally distributed, we have $\langle A, X\rangle \sim \mathcal{N}(0,\norm{X}_F^2)$. By \citet[Example 2.5.8 (a)]{vershynin2018high}, $\langle A, X\rangle$ has sub-gaussian norm at most $C_2\norm{X}_F$ for some universal constant $C_2$. Here the sub-gaussian norm of a random variable $x$ is defined as $\inf \{t > 0 \colon \E e^{x^2/t^2} \le 2\}$ \citep[Definition 2.5.6]{vershynin2018high}. It is clear from the definition that the absolute value $|x|$ has the same sub-gaussian norm as $x$. Hence $|\langle A, X\rangle|$ also has sub-gaussian norm at most $C_2\norm{X}_F$.

  By \citet[Corollary 7.3.3]{vershynin2018high}, we have $\mathbb{P}\{\norm{A} \ge \sqrt{n_1}+\sqrt{n_2}+t\} \le 2e^{-ct^2}$ for some universal constant $c$. By \citet[Proposition 2.5.2]{vershynin2018high}, this means $\norm{A}$ is sub-gaussian with sub-gaussian norm at most $C_3\sqrt{n_1+n_2}$ for some universal constant $C_3$.
\end{proof}

We will now derive a simple tail bound for sub-exponential matrices. It is a consequence of the following theorem from \citet[Theorem 3.6]{tropp2012user}:
\begin{theorem}\label{tropp_master}
  Consider a finite sequence $\{Z_i\}_{i=1}^m$ of independent, random, symmetric matrices. For all $t \in \R$, 
  \begin{align*}
    \mathbb{P}\left\{\lambda_1\left(\sum_{i=1}^m Z_i\right) \ge t\right\} \le \inf_{\theta > 0}\left\{e^{-\theta t}\cdot \mathrm{tr} \exp\left(\sum_{i=1}^m \log \E e^{\theta Z_i}\right)\right\}.
  \end{align*}
\end{theorem}

To use \Cref{tropp_master}, we need a bound on the MGF (Moment Generating Function) for sub-exponential matrices:
\begin{lemma}\label{sub_expo_mgf}
  Let $Z$ be a zero-mean, symmetric random matrix with sub-exponential norm at most $K$. Then for $0 \le \theta \le 1/K$, we have the following bound on the MGF:
  \begin{align*}
    \E \exp(\theta Z) \preccurlyeq \exp(K^2\theta^2) I.
  \end{align*}
\end{lemma}
\begin{proof}
  Without loss of generality, we may assume that $K = 1$. Then for $0 \le \theta \le 1$, we have

  \begin{align*}
    \lambda_1(\E \exp(\theta Z)) &\le \norm{\E \sum_{i=0}^\infty \frac{\theta^i}{i!} Z^i} = \norm{I+\E \sum_{i=2}^\infty \frac{\theta^i}{i!}Z^i}\\
    &\le 1+\E \sum_{i=2}^\infty \frac{\theta^i}{i!}\norm{Z^i} \le 1+\theta^2 \E\sum_{i=2}^\infty \frac{1}{i!}\norm{Z}^i\\
    &\le 1+\theta^2\left(-1+\E\sum_{i=0}^\infty \frac{1}{i!}\norm{Z}^i\right) = 1+\theta^2\left(-1+\E \exp(\norm{Z})\right)\\
    &\le 1+\theta^2\le \exp(\theta^2).
  \end{align*}
\end{proof}

We are now ready to state our general tail bound for sub-exponential matrices.
\begin{proposition}\label{sub_exp_bernstein}
  Let $\{Z_i\}^m_{i=1}$ be a sequence of independent, zero-mean, $n_1\times n_2$ random matrices with sub-exponential norms bounded by $K > 0$. Then for all $t \ge 0$,
  \begin{align*}
    \mathbb{P}\left\{\norm{\sum_{i=1}^m Z_i} \ge t\right\} \le (n_1+n_2) \exp\left(-\min\left(\frac{t}{2K}, \frac{t^2}{4K^2m}\right)\right).
  \end{align*}
\end{proposition}
\begin{proof}
  Consider the self-adjoint dilations $\{\hat Z_i\}_{i=1}^m$, where $\hat Z_i \coloneqq \begin{pmatrix}0 & Z_i\\Z_i\T & 0\end{pmatrix}$. It follows from properties of self-adjoint dilations that $\lambda_1(\hat Z_i) = \norm{Z_i}$. Hence, $\lambda_1(\hat Z_i)$ also has sub-exponential norm bounded by $K$. Moreover, $\hat Z_i$ is also zero-mean.

  Next, we apply \Cref{sub_expo_mgf}, which yields
  \begin{align*}
    \E \exp(\theta \hat Z_i) \preccurlyeq \exp(K^2\theta^2)I,
  \end{align*}
  for $0 \le \theta \le 1/K$.

  We use this to simplify the right hand side of \Cref{tropp_master} for the sequence $\{\hat Z_i\}_{i=1}^m$. This yields
  \begin{align*}
    \mathbb{P}\left\{\lambda_1\left(\sum_{i=1}^m \hat Z_i\right) \ge t\right\} &\le \inf_{0 < \theta}\left\{e^{-\theta t}\cdot \mathrm{tr} \exp\left(\sum_{i=1}^m \log \E e^{\theta \hat Z_i}\right)\right\}\\
                                                                               &\le (n_1+n_2) \inf_{0 < \theta \le 1/K}\exp(K^2m\theta^2 - \theta t)\\
                                                                               &= (n_1+n_2) \exp\left(-2m\mathcal{H}_1\left(\frac{t}{2mK}\right)\right)\\
                                                                               &\le (n_1+n_2) \exp\left(-\min\left(\frac{t}{2K}, \frac{t^2}{4Km^2}\right)\right).
  \end{align*}
  Here $\mathcal{H}_a(x) \coloneqq \begin{cases}x^2/2&\text{ for } |x| \le a,\\a |x|-a^2/2&\text{ otherwise, }\end{cases}\quad$ is the Huber loss function.

  Noting that $\sum_{i=1}^m \hat Z_i$ is a self-adjoint dilation, we have $\lambda_1\left(\sum_{i=1}^m \hat Z_i\right) = \norm{\sum_{i=1}^m Z_i}$, which concludes the proof.
\end{proof}

We finally apply the general tail bound \Cref{sub_exp_bernstein} to the measurement error in stochastic gradient descent for mini-batches of arbitrary size $m \ge 1$.
\begin{proposition}\label{sgd_E_bound}
  Fix an $n_1\times n_2$ matrix $X$. Let $S = \frac{1}{m}\sum_{i=1}^m\left(\langle A_i, X\rangle A_i\right)-X$ where $A_i \in \R^{n_1\times n_2}$ has i.i.d. entries with distribution $\mathcal{N}(0,1)$. Then, for $\xi \le c_1\sqrt m$, with probability at least $1-(n_1+n_2)e^{-c_2\xi^2}$, we have
  \begin{align*}
    \norm{S} \le \xi\sqrt{\frac{n_1+n_2}{m}}\norm{X}_F,
  \end{align*}
  where $c_1,c_2 > 0$ are universal constants.
\end{proposition}
\begin{proof}
  We may write $S = \frac{1}{m}\sum_{i=1}^m Z_i$ where $Z_i \coloneqq \langle A_i, X\rangle A_i-X$. \Cref{sub_exp_general} then says $\norm{Z_i}_{\phi_1} \le K \coloneqq C\sqrt{n_1+n_2}\norm{X}_F$ for some universal constant $C > 0$. Next, apply \Cref{sub_exp_bernstein} to the sum $m S$, which yields
  \begin{align*}
    \mathbb{P}\left\{\norm{mS} \ge t\right\} \le (n_1+n_2) \exp\left(-\min\left(\frac{t}{2K}, \frac{t^2}{4K^2m}\right)\right).
  \end{align*}
  Select $t = m \cdot \xi\sqrt{\frac{n_1+n_2}{m}}\norm{X}_F$. Then if $\xi \le 2C\sqrt m$, we have
  \begin{align*}
    \mathbb{P}\left\{\norm{S} \ge \xi\sqrt{\frac{n_1+n_2}{m}}\norm{X}_F\right\} &\le (n_1+n_2) \exp\left(-\min\left(\frac{\sqrt{m}\xi}{2C}, \frac{\xi^2}{4C^2}\right)\right)\\
                                                                                &\le (n_1+n_2) \exp\left(-\frac{\xi^2}{4C^2}\right).
  \end{align*}
\end{proof}

\begin{proposition}\label{EA_bound_sgd}
  Let $\bar R = Y-U V\T$, $R = \begin{pmatrix}0 & \bar R\\\bar R\T & 0\end{pmatrix}$, $W = \begin{pmatrix}U\\V\end{pmatrix}$ and
  \[E^\A = \frac{1}{m}\sum_{i=1}^m\left(\langle A_i, \bar R\rangle A_i\right)-\bar R,\] where $A_i \in \R^{n_1\times n_2}$ has i.i.d. entries with distribution $\mathcal{N}(0,1)$. Then, for $\xi \le c_1\sqrt m$, with probability at least $1-(n_1+n_2)e^{-c_2\xi^2}$, we have
  \[\norm{E^\A} \le \xi\sqrt{\frac{r(n_1+n_2)}{m}}\left(\sqrt 2\norm{R}+\sqrt{\frac{3\min(n_1,n_2)}{r}}\sigma_{r+1}^2(W)\right).\]
  The constants $c_1,c_2 > 0$ are universal.
\end{proposition}
\begin{proof}
  The result is a simple combination of \Cref{sgd_E_bound} with $X = \bar R$, and the bound \Cref{RF_bound}.
\end{proof}

\end{document}